\documentclass[twoside,11pt]{article}

\usepackage{blindtext}

\usepackage{jmlr2e}

\usepackage{lastpage}

\usepackage{microtype}
\usepackage{graphicx}
\usepackage{subfigure}
\usepackage{booktabs} 
\usepackage{bbold}
\usepackage{longtable}
\usepackage{supertabular}
\usepackage[font={small,it}]{caption}

\usepackage{hyperref}

\usepackage{geometry}
\usepackage{amsmath}
\usepackage{amssymb}
\usepackage{mathtools}
\usepackage{amsfonts}
\usepackage{tikz}
\usetikzlibrary{decorations.pathreplacing}

\usepackage[ruled,vlined]{algorithm2e}

\newcommand{\bbw}{\Bar{\Bar{w}}}
\newcommand{\bw}{\Bar{w}}
\newcommand{\Obj}{\mathcal{O}}

\newcommand{\Tsel}{T_{\text{sel}}}

\newcommand{\abs}[1]{\lvert #1 \rvert}
\newcommand{\Abs}[1]{\Big\lvert #1 \Big\rvert}
\newcommand{\AAbs}[1]{\left\lvert #1 \right\rvert}
\newcommand{\norm}[1]{\lVert #1 \rVert}
\newcommand{\Norm}[1]{\left\lVert #1 \right\rVert}

\newcommand{\brac}[1]{\langle #1 \rangle}
\newcommand{\Brac}[1]{\Big\langle #1 \Big\rangle}

\newcommand{\J}{\Bar{J}_n}
\newcommand{\Tr}{T_{\text{renorm}}}
\newcommand{\Tl}{T_{\text{learn}}}

\newcommand{\disc}{\text{disc}_M}
\newcommand{\Disc}{\text{Disc}_M}
\newcommand{\FDisc}{\Phi_{\text{disc}}}

\newcommand{\ie}{{\em i.e.,~}}
\newcommand{\R}{\mathbb{R}}

\newcommand{\E}{\mathbb{E}}

\newcommand{\h}{\mathcal{H}}
\newcommand{\N}{\mathbb{N}}

\DeclareMathOperator{\softmax}{softmax}
\usepackage[capitalize,noabbrev]{cleveref}

\newtheorem{assumption}[theorem]{Assumption}

\allowdisplaybreaks

\ShortHeadings{Learning in chemical reaction networks}{Sophie Jaffard and Ivo F.~Sbalzarini}
\firstpageno{1}

\begin{document}

\title{Chemical Reaction Networks Learn Better than Spiking Neural Networks}

\author{\name Sophie Jaffard \email jaffard@mpi-cbg.de \\
  \addr
  Max Planck Institute of Molecular Cell Biology and Genetics \\
  Center for Systems Biology Dresden \\
  Dresden, Germany
  \AND
  \name Ivo F.~Sbalzarini \email sbalzarini@mpi-cbg.de \\
  \addr Dresden University of Technology, Faculty of Computer Science\\
  Max Planck Institute of Molecular Cell Biology and Genetics \\
  Center for Systems Biology Dresden \\
Dresden, Germany}

\editor{My editor}

\maketitle

\begin{abstract}
  We mathematically prove that chemical reaction networks without hidden layers can solve tasks for which spiking neural networks require hidden layers. Our proof uses the deterministic mass-action kinetics formulation of chemical reaction networks. Specifically, we prove that a certain reaction network without hidden layers can learn a classification task previously proved to be achievable by a spiking neural network with hidden layers. We provide analytical regret bounds for the global behavior of the network and analyze its asymptotic behavior and Vapnik--Chervonenkis dimension. In a numerical experiment, we confirm the learning capacity of the proposed chemical reaction network for classifying handwritten digits in pixel images, and we show that it solves the task more accurately and efficiently than a spiking neural network with hidden layers. This provides a motivation for machine learning in chemical computers and a mathematical explanation for how biological cells might exhibit more efficient learning behavior within biochemical reaction networks than neuronal networks.
\end{abstract}


\section{Introduction}

Living cells process information through biochemical reaction networks~\citep{Bray1995}. This has been verified experimentally both in gene regulation networks, where different network topologies give rise to different logical functions~\citep{Guet2002}, and in cell signaling networks, where the temporal dynamics of signaling proteins encode information that cells decode into distinct fate decisions~\citep{Purvis2013}. A natural question is whether biochemical networks are capable of not only processing information but also \emph{learning} from past inputs. In biology, the concept of learning has so far mostly been reserved for electrical signals between neurons, mainly in the brain. However, it could be much more widespread, with every cell performing its own learning using chemically encoded signals. It is then interesting to compare learning in neuronal networks with learning occurring within cells through biochemical reaction networks.

The expectation that individual cells can learn via chemical signals is well-founded from a theoretical standpoint. Chemical reaction networks (CRNs) are known to be strongly Turing complete, even in the continuum mass-action formulation~\citep{Fages2017}. Learning is a computational task, so CRNs are in principle able to solve it. The interesting question is therefore not \emph{whether} CRNs can learn, but \emph{how and when}: under what structural conditions does a CRN exhibit learning behavior, and what are the necessary components? How can learning be mathematically defined and guaranteed both asymptotically and for finite time? How powerful are CRNs as learning machines, \ie how does their complexity scale in comparison to neuronal networks? Any reaction network that possesses the required components is then---in principle---capable of learning.

Several machine-learning models have already been successfully implemented as CRNs, providing initial indications that supervised learning may be feasible. For instance, \cite{hjelmfelt1991chemical} designed a CRN implementing a McCulloch--Pitts neuron, \cite{banda2013online} introduced a chemical perceptron, and further examples include Neural ODEs~\citep{nagipogu2025neural} and Boltzmann machines~\citep{poole2017chemical}. While these studies occasionally prove that the proposed CRNs correctly implement the intended models, they generally do not provide formal mathematical guarantees or definitions of learning behavior. Similarly, online learning frameworks such as multi-armed bandit algorithms and reinforcement learning methods have been used as external tools for the optimization and control of CRNs---for example, to predict reaction conditions~\citep{wang2024identifying} and to improve reaction performance~\citep{perez2020adaptivebandit}---but the explicit implementation of such algorithms \emph{within} CRNs themselves remains largely unexplored.

Here, we address this gap by introducing a CRN for which supervised learning is established through rigorous theoretical proof. Our construction rests on a structural analogy between stochastic chemical kinetics and stochastic models of spiking neural networks (SNNs) and thus enables comparing the two learning frameworks. In stochastic chemical kinetics, reactions between molecular species occur randomly at rates that depend on species counts; each reaction event updates these counts and thereby alters the rates of subsequent reactions. SNNs can be mathematically modeled using multivariate Hawkes processes~\citep{hawkes1971spectra,gerhard2017stability,bonnet2022neuronal,lambert2018reconstructing}, in which each neuron emits spikes according to a stochastic intensity that depends on the past activity of its presynaptic neurons. The connection between the two frameworks is particularly apparent at the level of simulation: Gillespie's algorithm~\citep{gillespie1976general,gillespie1977exact} for CRNs is closely related to Ogata's thinning algorithm~\citep{ogata1981lewis} for Hawkes processes. Building on this analogy, our CRN is inspired by the SNN introduced by~\citet{jaffard2026chani}, for which learning guarantees have been established.

Even though the analogy is inspired by stochastic chemical kinetics, we here consider the continuous mass-action limit in order to support deterministic computation. In this formulation, each chemical input and output species functions as a neuron, with continuous concentrations playing the role of firing rates. This results in a simple ordinary differential equation description of the CRN, for which we derive theoretical guarantees that are comparable to those known for the SNN. Surprisingly, the CRN turns out to be a more powerful learning machine than the SNN: we show that it achieves the same performance bounds without requiring species analogous to hidden neurons. Instead, in a CRN, information is directly transmitted from input to output species, and the learning guarantees hold under weaker assumptions. We hypothesize that this is because multiplication is intrinsic in the physics of mass-action kinetics, whereas it needs to be approximated by sequences of layer-wise operations in neuronal networks.

Concretely, the CRN proposed here solves a supervised classification task in which items are to be assigned to one of several classes based on environmental information encoded in the concentrations of catalytic input species, while classes are represented by output species. The system evolves through two consecutive phases. In the \emph{selection phase}, sets of input species with high flux are selected. In the subsequent \emph{learning phase}, the selected input species act as catalysts and, together with designated weight species, drive the production of output species through reactions analogous to the layered computations of a SNN. The weight species play the role of connection weights and evolve according to an expert aggregation algorithm~\citep{cesa2006prediction}: each set of selected input species acts as an expert providing information about the environment to each output species. Weight concentrations are updated based on the evolution of the expert fluxes over time. Importantly, species associated with different classes evolve independently, interacting neither directly nor indirectly through chemical reactions.

This system allows us to establish the following mathematical results: first, we derive local regret bounds for each output species (Proposition~\ref{prop reg}), establishing that each output species locally behaves so as to have higher-than-average concentration when presented with an item belonging to the class it encodes. Second, we establish an oracle inequality for the global behavior of the network (Theorem~\ref{th oracle}), demonstrating that the average network performance during learning is asymptotically optimal. Since output species behave independently of one another, this contributes to a theoretical understanding of emergence in systems in which global behavior arises from local interactions among components. Third, we compute the asymptotic weight concentrations in a more restricted setting, where we are able to give convergence rates (Theorem~\ref{theo lim wk}), and characterize the classes that the network can learn under these asymptotic weights (Theorem~\ref{th equiv}). We then determine the VC-dimension of the network in this setting (Proposition~\ref{prop vcdim}) as a task- and model-complexity measure. In Section~\ref{sec discussion}, we discuss the role of each component in these theoretical results, highlighting the key mechanisms that enable the emergence of learning behavior in CRNs. Finally, in Section~\ref{sec num res}, we present numerical experiments on handwritten digits data set,
showing that the deterministically simulated CRN achieves satisfactory test accuracy in low-complexity configurations and outperforms SNNs of comparable model complexity. This also experimentally confirms that CRNs without hidden layers can be more powerfull than SNNs with hidden layers.

\section{Problem setting}\label{sec sup learn}

A chemical species is denoted by a capital letter $X$, and its concentration by the corresponding lowercase letter $x$. For a quantity $x_e$ indexed by $e\in E$, we use the following notation: $x_E = (x_e)_{e\in E}$ and $\brac{x_e}_{e\in E} = \frac{1}{\abs{E}}\sum_{e\in E} x_e$. This notation is also used if the index is in superscript. For $n\in \N^*$, we denote $[n]=\{1,\dots,n\}$. For a finite set $E$, we denote by $\mathcal{P}_E$ the set of probability distributions over the set $E$. We study our proposed CRN in the framework of mass-action kinetics and deterministic rate equations.

We consider a supervised classification setting where the goal is to assign items to classes $k\in K$ using a CRN. More precisely, the CRN learns through sequential exposure to $M$ training samples. Each training sample $o_m$ for $m\in [M]$ alters the concentrations of catalytic chemical species $(X^i)_{i\in I}$ in a given time interval $[T^0_m, T^1_m]$: these species can be interpreted as {\em input species} describing features of the presented sample $o_m$, and the set $I$ represents the set of input species as well as the set of features they encode. In other words, the CRN aims to learn a mapping from a set of concentrations $(x^i)_{i\in I}$ in a given time interval to the set of classes $K$.

For example, in a biological cell, a gene regulatory network might be tasked with distinguishing among a set $K$ of possible cell states. In this case, a sample $o_m$ corresponds to measured environmental concentrations, the set $I$ contains relevant transcription factors, and the $x^i$ are their concentrations.  In a conventional image-classification example, the $o_m$ would be images, where each concentration $x^i$ represents the intensity of a given pixel.

For each presented sample $o_m$, the decision of the CRN is encoded in the concentrations of the chemical species $X^k$ associated with each class $k \in K$. The sample $o_m$ is assigned to the class corresponding to the species $X^k$ with highest concentration at time $T^1_m$. The species $(X^k)_{k\in K}$, which we call {\em output species}, interact with the input species $(X^i)_{i\in I}$ through {\em intermediary species}, which we introduce in Section~\ref{sec desc CRN}.

\subsection{Network architecture}
The depth of the CRN is defined by an integer $n\in \{1,\dots, \abs{I}\}$. Let $J_n$ be the set of all subsets of $I$ of size $n$. For certain sets $j\in J_n$, our CRN has equations involving the term $\sum_{i\in j} X^i$.
We then define the flux of a set $j$ as follows.

\begin{definition}[Flux of a set $j$]
  The flux of a set $j\in J_n$ is the product $\Phi^j := \prod_{i\in j} x^i$. It represents the flux (total propensity) of a reaction with $\sum_{i\in j} X^i$ as its left (reactant) side.
\end{definition}

The flux $\Phi^j$ can be interpreted as the strength of the combination of all features $i\in j$. For instance, in the context of a gene regulatory network, a high value of $\Phi^j$ at time $t$ indicates that the transcription factors associated with set $j$ are present at elevated concentrations simultaneously at this time. This condition may reflect coordinated activation of multiple genes. The larger the value of the depth $n$, the more complex the combination of features captured by the flux $\Phi^j$ for the input species $j$.

In order to allow for supervised training, our CRN implements an online learning framework using expert aggregation \citep{cesa2006prediction}. We consider the classic expert aggregation problem: over $M$ rounds, a forecaster has access to a set $E$ of experts. At each round $m\in[M]$, every expert $e \in E$ obtains an unknown gain $g_m^e \in \mathbb{R}$. Based on past observations, the forecaster chooses a probability distribution $p_m \in \mathcal{P}_E$ over the experts and subsequently receives the aggregated gain $g_m := p_m \cdot g_m^E$, where $g_m^E := (g_m^e)_{e\in E}$ and $\cdot$ is the standard scalar product on $\R^{\abs{E}}$.
This aggregated gain can be interpreted as the expected gain that the forecaster would obtain when selecting an expert at random according to $p_m$. Experts accumulate gains
$G_m^E := (G_m^e)_{e\in E}$ where $G_m^e := \sum_{m'=1}^m g_{m'}^e$, and the forecaster similarly accumulates $G_m := \sum_{m'=1}^m g_{m'}.$

The {\em regret} of the forecaster quantifies how far its strategy is from optimality. It is defined as:
\[
  R_M :=
  \max_{q \in \mathcal{P}_E}
  \sum_{m=1}^M q \cdot g_m^E - G_m \, .
\]
Thus, $R_M$ compares the best cumulative gain obtained by a fixed distribution over experts with the cumulative gain actually achieved by the forecaster.

In the present context, an expert aggregation algorithm is defined by a map $\varphi: \R \to \R_+$ such that for each $m$, $\max_{e\in E} \varphi(G^e_m) >0$. The rule to update the forecaster's probability distribution over experts $e$ is
\begin{equation} \label{eq phi}
  p^e_{m+1} = \frac{\varphi(G^e_m)}{\sum_{e'\in E}\varphi(G^{e'}_m) } \, .
\end{equation}
This guarantees {\em regret bounds} of order $\sqrt{M}$. Here, we use the classic EWA algorithm.
\medskip

\noindent \textbf{EWA (Exponentially Weighted Average)} \citep{littlestone1994weighted}.
The probability of selecting expert $e$ at round $m+1$ is
\[
  p_{m+1}^e= \frac{\exp(\eta G_m^e)}
  {\sum_{e'\in E} \exp(\eta G_m^{e'})} \, ,
\]
where $\eta > 0$ is the learning rate, which determines the sensitivity of the forecaster to past performance. See Appendix \ref{appendix EWA} for details on the regret bound achieved by this algorithm for different choices of $\eta$.

\section{Description of the chemical reaction network} \label{sec desc CRN}

During training, the CRN evolves over two consecutive phases: a selection phase and a learning phase. The selection phase performs unsupervised feature selection, identifying informative feature subsets and adjusting the network connectivity accordingly. The subsequent learning phase adjusts the concentrations of special weight species in the network using labeled training data.
These phases, along with the reactions occurring during them, are illustrated in Figure \ref{fig:CRN}. The mass-action rate equations for the network are given in Appendix \ref{appendix rate eq}.

\begin{figure}
  \centering
  \includegraphics[width=1\linewidth]{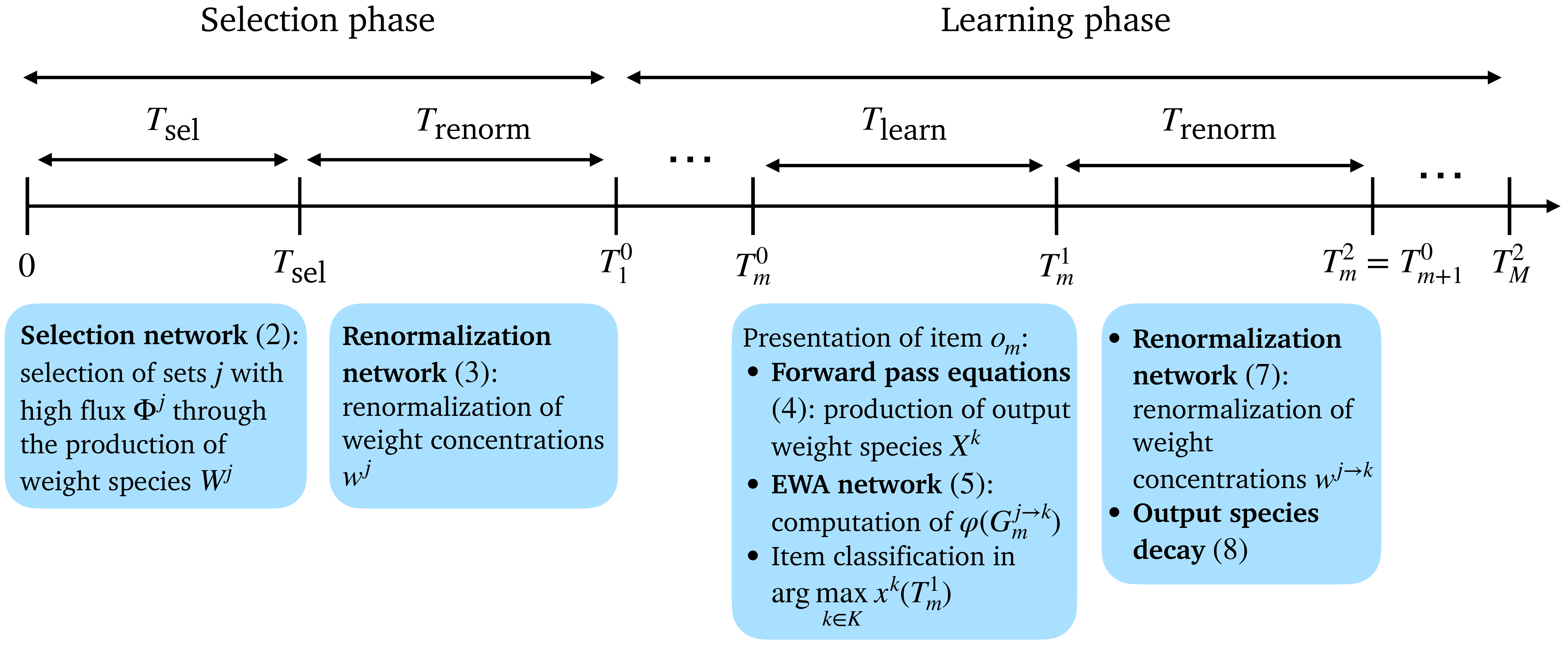}
  \caption{Schematic representation of the time evolution of the chemical reaction network during the selection and learning phases.}
  \label{fig:CRN}
\end{figure}

\subsection{Selection phase}

The selection phase is divided in two parts. The aim of the first part is to induce the production of chemical species $W^j$ only for sets $j \in J_n$ that exhibit a significant flux $\Phi^j$ during a portion of a time interval $[0, \Tsel]$. For example, during $[0, \Tsel]$, the network may be exposed sequentially to a certain number of training samples that may or may not be incorporated into the learning phase. The species $W^j$ are called weight species. The concentrations $(x^i)_{i\in I}$ are assumed to be continuous on $[0,\Tsel]$. The reactions occurring during the selection phase are called the {\em selection network} on $[0,\Tsel]$, defined as:
\begin{align}
  \forall  j\in J_n, \quad \sum_{i\in j} X^i & \xlongrightarrow{f} W^j + \sum_{i\in j} X^i  \label{eq Wj} \, ,
\end{align}
where $f$ is a continuous sigmoidal function defined on $\R_+$ such that there exist $ \theta, \rho >0$ with $f\equiv 0$ on $[0, \theta]$, $f$ is increasing on $(\theta, \theta +\rho)$, and $f \equiv \norm{f}_{\infty}$ on $[\theta + \rho,+\infty)$. The rate of this reaction at time $t\in [0,\Tsel]$ is $f(\Phi^j(t))$. Let
\[
  \J := \left\{j\in J \,\,\, | \,\,\, \exists [t_0,t_1] \subset [0,\Tsel] \text{ with }  t_0 < t_1 \text{ such that } \Phi^j > \theta  \text{ on } [t_0,t_1] \right\} \, .
\]
Reaction \eqref{eq Wj} happens only for indexes $j\in \J$, where the set $\J$ contains selected sets of species with fluxes surpassing the threshold $\theta$ during at least a sub-intervall of the selection phase, and for which weight species $W^j$ are produced. Therefore, the selection network is composed of $\abs{\J}$ reactions.

During the second part of the selection phase, the concentrations of the produced weight species are renormalized over a time interval $[\Tsel,T^0_1]$ ($\Tr := T^0_1 - \Tsel$) according to the following set of reversible reaction equations that we call the {\em renormalization network}:
\begin{equation}
  \begin{aligned}
    \forall j\in\J\quad & A \xlongrightarrow{b_1} A + W^j \\
    \forall j\in\J\quad & A + W^j \xlongrightarrow{b_2} A
  \end{aligned}
  \label{eq renorm wj}
\end{equation}
using a catalytic species $A$ whose concentration stays constant at a positive initial value.

\subsection{Learning phase}
The selection phase is followed by a learning phase during which the network learns to perform the classification task. The learning phase occurs during the time interval $[T^0_1,T^2_M]$, subdivided into $2M$ intervals $T^0_1 < T^1_1 < T^2_1 = T^0_2 < \dots <T^0_m < T^1_m < T^2_m = T^0_{m+1} < \dots < T^0_M < T^1_M < T^2_M$.  All the intervals $[T^0_m, T^1_m]$ (resp. $[T^1_m, T^2_m]$) have the same length $\Tl$ (resp. $\Tr$). See Figure \ref{fig:CRN} for an illustration.

At each iteration $m\in [M]$ of the learning phase, a training sample $o_m$ is presented during the time interval $[T^0_m, T^1_m]$, at the end of which the network classifies the sample into the class $\arg\max_{k\in K} x^k(T^1_m)$.
This involves the input species $(X^{i})_{i \in I}$, the weight species $(W^{j})_{j \in J}$, the output species $(X^{k})_{k \in K}$, as well as several additional species. Output species start with initial concentrations equal to zero. The additional species include the species $(W^{j \to k})_{j \in J, k \in K}$, referred to as output weight species. They quantify the influence of a set of input species $j\in \J$ on the output species $k$. Further additional species are the gain-function species $(H^{j \to k})_{j \in \J, k \in K}$, which record the sensitivity of species $i$ to class $k$ and are used to implement the expert-aggregation EWA algorithm.
All $H^{j \to k}$ share a common positive initial concentration, and the $W^{j \to k}$ share a common initial concentration equal to $\frac{b_{1}}{b_{2}\abs{\J}}$. The concentrations $(x^i)_{i\in I}$ are assumed to be continuous on $[T^0_m,T^1_m]$. This yields the following {\em forward-pass learning reactions} on $[T^0_m, T^1_m]$:
\begin{align}
  & \forall j\in \J, k\in K, \quad  W^j + W^{j\to k} + \sum_{i\in j} X^i  \xlongrightarrow{a_1} X^k +W^{j\to k} + W^j + \sum_{i\in j} X^i \, . \label{eq net}
\end{align}
These reactions are called ``forward pass'' due to their analogy with classic forward equations of artificial neural networks and SNN, describing how the combination of the weights $(W^j)_{j\in \J}$ and $(W^{j\to k})_{j\in \J, k\in K}$ and inputs $(X^i)_{i\in j}$ impacts output $X^k$. Chemically, the weight and input species act as catalysts. These reactions are illustrated in the right half of Figure \ref{fig:CRN SNN}.

Simultaneously to the forward-pass reactions, the reactions corresponding to the EWA network occur. Let $k^\star$ be the true class of training sample $o_m$. For $k\in K$, let $d_k := \eta \frac{M}{M^k}$ and $c_k := \eta \frac{M}{M^k}\times \frac{1}{\abs{K}-1}$, where $\eta >0$ is a constant and $M^k$ is the number of training samples that belong to class $k$ from a total of $M$ samples. The following reactions then implement the {\em EWA network} on $[T^0_m, T^1_m]$:
\begin{equation}
  \begin{aligned}
    \forall j\in \J, \quad &  H^{j\to k^\star} + W^j + \sum_{i\in j} X^i \xlongrightarrow{d_{k^\star}}  2 H^{j\to k^\star} + W^j + \sum_{i\in j} X^i \, , \\
    \forall j\in \J, k\neq k^\star,  \quad &  H^{j\to k} +W^j +\sum_{i\in j} X^i\xlongrightarrow{c_{k}}W^j +\sum_{i\in j} X^i \, , \\
    \forall j\in \J,  \quad & S + H^{j\to k^\star} \xlongrightarrow{d_{k^\star}} S + 2 H^{j\to k^\star} \, , \\
    \forall j\in \J, k\neq k^\star \quad & S + H^{j\to k} \xlongrightarrow{c_{k}} S + 2 H^{j\to k} \, .
  \end{aligned}
  \label{eq crn EWA}
\end{equation}
The catalytic species $S$ is included to ensure that the concentrations of the gain-function species remain positive.
Each output species $X^k$ is a forecaster, and the experts are the sets of input species $j\in \J$. The gain $g^{j\to k}_m$ of the set of input species $j$ w.r.t.~output species $k$ during round $m$ is
\begin{equation*} \label{eq gain output}
  g^{j\to k}_m :=
  \begin{cases}
    \left( w^j(T^1_0) \int_{T^0_m}^{T^1_m} \Phi^j + s_0\Tl \right)\times \frac{M}{M^k} &\text{if } o_m\in k \, ,   \\
    \left( - w^j(T^1_0)\int_{T^0_m}^{T^1_m} \Phi^j+ s_0\Tl\right) \!\times\! \frac{M}{M^{k'}} \!\times\! \frac{1}{\abs{K}-1} &\text{if } o_m \in k'\neq k \, .
  \end{cases}
\end{equation*}
Therefore, if the presented training sample $o_m$ belongs to class $k$, then species $X^k$ attributes a positive gain to the set of input species $j$ that grows linearly with its integrated flux. Conversely, if the presented sample $o_m$ belongs to another class, then species $X^k$ attributes a negative gain, \ie a loss, to the set $j$ that decreases linearly with its flux.
Then, the cumulated gain of the set $j$ w.r.t.~class $k$ is
\begin{equation} \label{eq def cum gain}
  G^{j\to k}_m := \sum_{m'=1}^M g^{j\to k}_{m'} \, ,
\end{equation}
and the rate equations associated with the EWA network \eqref{eq crn EWA} (see Appendix \ref{appendix rate eq} for details) provide that $h^{j\to k}(T^1_m)=\varphi(G^{j\to k}_m)$, where $\varphi : x \to \exp(\eta x)$ is the function in Eq.~\eqref{eq phi}.
Note that, in principle, it is possible to implement expert aggregation other than EWA by suitably modifying the reactions in Eq.~\eqref{eq crn EWA}.

Just like in the selection phase, the learning phase is also followed by weight renormalization during the time interval $[T^1_m, T^2_m]$. During this time interval, the following {\em renormalization network} is active:
\begin{equation}
  \begin{aligned}
    \forall j\in \J, k\in K, \quad & H^{j\to k} \xlongrightarrow{b_1}  H^{j\to k} + W^{j\to k} \\
    \forall j, j'\in \J, k\in K \quad &  H^{j'\to k} + W^{j\to k} \xlongrightarrow{b_2} H^{j'\to k} \, .
  \end{aligned}
  \label{eq renorm wk}
\end{equation}
Similarly to Eq.~\eqref{eq renorm wj}, this renormalizes the output-weight species so that, at equilibrium, for every $k \in K$ we have $\sum_{j \in \J} w^{j \to k} = \frac{b_1}{b_2}$. In addition, during the same time interval $[T^1_m, T^2_m]$, the {\em output species decay} as
\begin{align} \label{eq decay xk}
  \forall k\in K, \quad & X^k \xlongrightarrow{a_2} \emptyset \, .
\end{align}
This resets the output species' concentrations to be close to zero before the beginning of the training round for the next training sample $m+1$.

Note that the equations of the learning phase are designed such that each output species $X^k$ receives information from the input species (forward-pass equations \eqref{eq net}) and runs its own expert aggregation algorithm through the EWA network \eqref{eq crn EWA} and renormalization network \eqref{eq renorm wk}, independently from the other output species: there is no equation involving species $X^k$, $W^{j\to k}$, or $H^{j\to k}$ from two different classes. Therefore, this CRN can be seen as a collection of \textbf{local equations} for each class $k$.

\subsection{Inference phase} \label{sec inference}

After the learning phase, the CRN can perform inference on new data by fixing the output weight concentrations to their final values $(w^{j\to k}(T^2_M))_{j\in \J, k\in K}$ and evolving only according to the forward-pass Eq.~\eqref{eq net} when presented with a new data sample, followed by the output species decay in Eq.~\eqref{eq decay xk} between two sample presentations. In the numerical results reported in Section \ref{sec num res}, this procedure is used to evaluate the network's accuracy on test data.

\subsection{Network complexity} \label{sec net complexity}

The total number of reactions for the selection and learning phases is $\abs{\J}(3+3\abs{K})$. These reactions involve $2+\abs{I} + \abs{\J} + \abs{K} + 2 \abs{\J}\abs{K}$ chemical species. Consequently, the choice of the threshold $\theta$, which appears in the rate of the selection network \eqref{eq Wj} and defines the set $\J$, is crucial. Increasing $\theta$ reduces $\abs{\J}$; however, as will be shown in Proposition \ref{prop vcdim}, the value of $\abs{\J}$ directly governs the complexity of the classes that the CRN can learn and, therefore, its statistical power. This defines the usual trade-off between network complexity and the capacity to represent complex class boundaries.

The total duration of each reaction is also important for the computational cost of implementing such a network. The selection network \eqref{eq Wj} operates for a duration $\Tsel$, and the renormalization network \eqref{eq renorm wj} for $\Tr$. The forward pass in Eq.~\eqref{eq net} and the EWA network in Eq.~\eqref{eq crn EWA} operate during $M$ time intervals of duration $\Tl$, hence for a total duration of $M\Tl$. Finally, the renormalization network \eqref{eq renorm wk} and the output species decay \eqref{eq decay xk} happen during $M$ time intervals of length $\Tr$, for a total duration of $M\Tr$.

To obtain the classification results, it is necessary to account for the concentrations of all output species at the end of each sample presentation. These concentrations evolve according to the forward-pass equations \eqref{eq net}. In addition, the weight concentrations $w^{j\to k}$ must be renormalized after each presented training sample by implementing the renormalization network \eqref{eq renorm wk}.

However, if the network is to be evaluated only after training has concluded, \ie for inference, then it is not necessary to produce output species during the learning phase, nor to renormalize the weights at each step $m$. In this case, it is sufficient to run the renormalization network \eqref{eq renorm wk} once, after all learning passes, for a single duration $\Tr$. The output species can then be generated using the forward-pass reactions \eqref{eq net} for an arbitrary duration corresponding to the presentation of test samples, as described in Section \ref{sec inference}. Within this simplified network, the asymptotic analysis presented in Section \ref{sec asymp analysis} remains valid.

\subsection{Analogy with spiking neural networks} \label{sec analogy SNN}

The proposed CRN architecture is inspired by the spiking neural network (SNN) introduced by \citet{jaffard2026chani}. The analogy between the two models is illustrated in Figure \ref{fig:CRN SNN}. In those SNN, the input layer consists of neurons that encode a set of features $I$ describing the training samples. The hidden layers are organized such that each hidden neuron $j$ represents correlations between two presynaptic neurons. Consequently, a neuron in the final hidden layer effectively encodes correlations among $2^L$ input neurons, where $L$ is the number of hidden layers. The output layer contains neurons coding for each class.

Every hidden layer initially contains neurons encoding all possible pairs of presynaptic neurons. The layer is subsequently trained using an expert-aggregation algorithm to detect correlations between  neuron pairs, after which neurons whose encoded correlations are below a certain threshold are pruned. Once all hidden layers have undergone training and pruning, the output layer is trained, again using expert aggregation, to perform the classification task. The analysis by \cite{jaffard2026chani} provides theoretical guarantees on the learning performance of such SNN.

In the present CRN setting, each input and output species $(X^l)_{l \in I \cup K}$ functions analogously to a SNN neuron, with its
concentration interpreted as the firing rate. The set $\J$ plays the role of the final hidden layer in the SNN. However, information is encoded more efficiently in the CRN than in the SNN: there is no need to introduce species corresponding to hidden neurons, as the input species directly influence the output species through the forward-pass equations~\eqref{eq net}. The selection mechanism is also more efficient. Rather than enumerating every possible subset $j$ of input species of size $n$ and pruning the outputs, the CRN begins empty, and whenever
a flux $\Phi^j$ exceeds the threshold $\theta$, a weight species $W^j$ is produced. This efficiency arises from the intrinsic capacity of CRNs to implement product operations: chemical mass-action reactions containing the term $\sum_{i \in j} X^i$ yields a flux $\Phi^j = \prod_{i \in j} x^i$ in its propensity. In contrast, SNNs do not have a direct mechanism for computing products of presynaptic activities. Moreover, the CRN network depth $n$, which corresponds to the size of the correlations encoded in the last hidden layer in the SNN, may be any integer and is not restricted to powers of two as in the SNN case.
In section \ref{sec th res}, we show that thanks to these differences, the same learning guarantees can be proven for CRN as for SNN, but that the present CRN achieves them more efficiently and within a less restrictive framework of assumptions.

\begin{figure}
  \centering
  \includegraphics[width=1\linewidth]{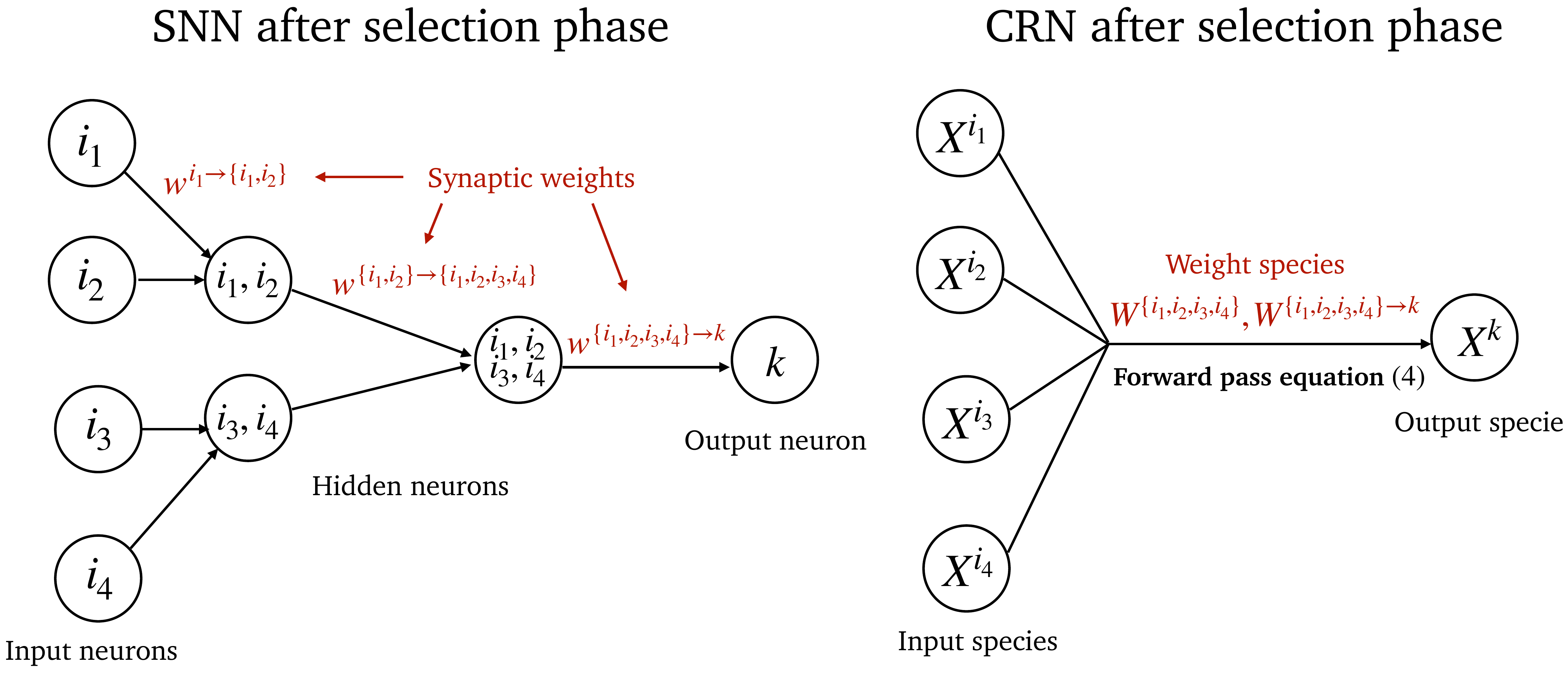}
  \caption{Analogy between the structure of the SNN proposed by \cite{jaffard2026chani} (left) and the present CRN (right).}
  \label{fig:CRN SNN}
\end{figure}

\subsection{Links to artificial neural networks}

The presented CRN and the learning task that it is designed to solve are deeply rooted in the field of machine learning, and the supervised classification setting described above is canonical.
Moreover, since artificial neural networks (ANNs) and spiking neural networks (SNNs) are closely related and have historically informed one another, the proposed CRN exhibits structural similarities to a conventional ANN. In particular, the input and output species function analogously to input and output neurons, while the weight species play a role comparable to connection weights (see Figure~\ref{fig:CRN SNN}). The forward-pass reactions~\eqref{eq net} therefore resemble the computation performed across multiple layers of an ANN.

Machine learning methods are also central to our theoretical analysis. In particular, we use regret bounds associated with the expert-aggregation algorithm implemented within our CRN to establish guarantees on classification accuracy. Additionally, in Theorem~\ref{th equiv}, we establish a result that is analogous to the Perceptron Convergence Theorem~\citep{rosenblatt1962principles}.

\section{Performance measures}
We derive guarantees about the above CRN's ability to solve the classification task defined in Section \ref{sec sup learn} in terms of local and global discrepancy measures.

\begin{definition}[Species discrepancy]
  The species discrepancy of species $X^k$ is defined as
  \[
    \disc^k := \Brac{x^k(T^1_m)}_{m, \ o_m\in k} - \Brac{x^k(T^1_m)}_{
      \begin{subarray}{l}
        k'\neq k \\
        m, \ o_m \in k'
    \end{subarray}}.
  \]
\end{definition}

The species discrepancy measures the difference between the average concentration of species $X^k$ at the end of the presentation of a training sample of its own class, and its average concentration at the end of the presentation of a sample of one of the other classes. This provides information about how much the concentration of species $X^k$ is larger than average after being exposed to a training sample of its own class. It is a {\bf local measure}, because it involves only the species $k$ itself.

A useful variant of the species discrepancy is the discrepancy with constant weights.

\begin{definition}[Species discrepancy with constant weights]
  For a family $q^k\in \frac{b_1}{b_2} \mathcal{P}_{\J}$, the concentration discrepancy of species $X^k$ with output weight concentrations $q^k$ is defined as
  \begin{align*}
    \disc^k(q^k) &:= \Brac{\frac{a_1b_1}{b_2}\int_{T^0_m}^{T^1_m} \sum_{j\in \J} q^{j\to k} \Phi^j(t) \,\mathrm{d}t}_{m, \ o_m\in k} \\
    & - \Brac{\frac{a_1 b_1}{b_2}\int_{T^0_m}^{T^1_m} \sum_{j\in \J} q^{j\to k}  \Phi^j(t) \,\mathrm{d}t}_{
      \begin{subarray}{l}
        k'\neq k \\
        m, \ o_m \in k'
    \end{subarray}}.
  \end{align*}
\end{definition}
This is the species discrepancy that $X^k$ would have if the weight concentrations $w^k = (w^{j\to k})_{j\in \J}$ were all equal to $q^k$, the input weight concentrations $w^j$ were all at equilibrium $\frac{b_1}{b_2}$, and if the initial concentrations $x^k(T^0_m)$ were zero for every $m$. This quantity measures the performance of a constant-weight family $q^k$ in recognizing samples of class $k$ in the ideal case where the renormalization network \eqref{eq renorm wj} and output species decay equation \eqref{eq decay xk} have reached equilibrium. It therefore provides a baseline for the species discrepancy of $X^k$, the latter depending on the network's weight concentrations $w^j$ and $w^k$.

To compare the overall classification performance of the network, one has to compare the concentrations of all output species. This is a global measure, which we call network discrepancy.

\begin{definition}[Network discrepancy]
  The network discrepancy is defined as
  \[  \Disc :=  \Brac{x^{k^\star}(T^1_m) - x^k(T^1_m)}_{
      \begin{subarray}{l}
        k^\star\in K \\
        k\neq k^\star \\
        m, \ o_m \in k^\star
    \end{subarray}}.
  \]
\end{definition}
The network discrepancy measures the average difference between the concentration of a species $X^{k^\star}$ and the concentration of the other output species at the end of the presentation of a training sample of class $k^{\star}$. A network discrepancy larger than 1 means that, on average, the concentration of the species coding for the correct class is higher than the concentrations of the other output species. Therefore, this measure quantifies the classification contrast of the CRN, providing a \textbf{global measure} that involves the concentration of all output species $X^k$ for $k\in K$.

Similar to above, one can also define the version with constant weights.

\begin{definition}[Network discrepancy with constant weights]
  For a family $q^K\in (\frac{b_1}{b_2} \mathcal{P}_{\J})^{\abs{K}}$, the network discrepancy with constant weights $q^K$ is defined as
  \[  \Disc(q^K) :=   \frac{a_1 b_1}{b_2}\Brac{\int_{T^0_m}^{T^1_m} \sum_{j\in \J} q^{j\to k^\star} \Phi^j(t) \,\mathrm{d}t - \int_{T^0_m}^{T^1_m} \sum_{j\in \J} q^{j\to k} \Phi^j(t) \,\mathrm{d}t}_{
      \begin{subarray}{l}
        k^\star\in K \\
        k\neq k^\star \\
        m, \ o_m \in k^\star
    \end{subarray}}.
  \]
\end{definition}
This represents the discrepancy that the CRN would have if the weight concentrations $w^K = (w^k)_{k\in K}$ were all equal to $q^K$, the input weight concentrations $w^j$ were at equilibrium $\frac{b_1}{b_2}$, and if the initial concentrations $x^k(T^0_m)$ were zero for all $m$.
\newline

Some of the theoretical results below are derived for a more specific setting of the CRN fulfilling the following assumption.

\begin{assumption}[Repetitive training samples] \label{assump natures and nb obj}
  There exists a set of sample types $\Obj$ such that the family of training samples $(o_m)_{1\leq m \leq M}$ consists of repetitions of types $o\in \Obj$: for every $m\in [M]$, there exists $o\in \Obj$ such that $o_m = o$, and two training samples $m_1$ and $m_2$ can share the same type $o_{m_1} = o_{m_2} = o\in \Obj$. Moreover, during the learning phase, each type of sample $o\in \Obj$ is presented to the network equally many times.
\end{assumption}

For instance, in the context of image classification, this assumption means that the network is learning to classify a (possibly small) set of images $\Obj$ by seeing them repeatedly during training. This is a generalization of the usual epoch-wise training in ANNs, where each sample is presented multiple times over epochs.
Consequently, under this assumption, for each object type $o\in \Obj$ the concentrations $(x^i)_{i\in I}$ are the same in every interval $[T^0_m, T^1_m]$ such that $o_m = o$. For a subset $j\in \J$ of input species, we then denote by $\int_o \Phi^j$ the common value of all the integrals $\int_{T^0_m}^{T^1_m} \Phi^j(u) \,\mathrm{d}u$ where $o_m = o$.

Under Assumption \ref{assump natures and nb obj}, we define the notion of flux discrepancy as follows.

\begin{definition}[Flux discrepancy] \label{def feat disc}
  The flux discrepancy of a set $j\in \J$ of input species w.r.t. class $k\in K$ is defined as
  \[
    \FDisc^{j\to k} := \Brac{\int_o \Phi^j}_{o\in k} - \Brac{\int_o \Phi^j}_{
      \begin{subarray}{l}
        k'\neq k \\
        o\in k'
    \end{subarray}} \, .
  \]
\end{definition}
This compares the average flux of a set of input species $j$ when presented with a sample of class $k$ to the average flux when presented with samples of other classes. The flux discrepancy thus indicates the extent to which the set of input species $j$ is sensitive to samples of class $k$.

\subsection{Optimal weight family}
Under Assumption \ref{assump natures and nb obj}, we define optimal output weight concentration families to which we can compare the weights learned by the CRN.

\begin{definition}[Optimal weight family]
  Let $q^K= (q^k)_{k\in K}\in \frac{b_1}{b_2}\mathcal{P}_{\J}$ be an output weight family on the whole network.
  Then, $q^K$ is said to be an optimal family if
  \[
    \forall k\in K, o\in \Obj, \quad \sum_{j\in \J} q^{j\to k} \int_o \Phi^j>0 \quad \text{if and only if} \quad o\in k.
  \]
\end{definition}

In other words, an optimal weight family is a family
of output weight concentrations such that each output species $X^k$ encodes exactly one class $k$: its concentration is positive if and only if the presented sample belongs to class $k$. This requirement is stronger than assuming that a weight concentration family enables to correctly classify every sample, as it additionally imposes that the concentration $x^k$ remain exactly zero whenever the presented sample does {\em not} belong to class $k$.

\section{Theoretical results} \label{sec th res}

We prove theoretical guarantees for the CRN's ability to solve the classification task defined in Section \ref{sec sup learn}.
This includes the derivation of error bounds that depend on the network parameters. For ease of notation, however, we make explicit only the dependence on the cardinalities $\abs{I}$, $\abs{\J}$, and $\abs{K}$, which determine the number of reactions and species in the system, as well as on the renormalization network duration $\Tr$ and the number of training samples $M$. We show that the error bound vanishes if these last two parameters are large. The time durations $\Tsel$ and $\Tl$ are treated as user-defined fixed constants. They are not required to be asymptotically large, and their dependence is thus omitted. The same applies to the reaction rate constants and the initial concentrations. All terms are, however, made explicit in the proofs available in Appendix \ref{sec:proofs}.
In section \ref{sec discussion}, we discuss the role of each component of the CRN in deriving the theoretical results.

\subsection{Selection phase}

We start by proving that the renormalization network \eqref{eq renorm wj} of the selection phase achieves its purpose, which is to renormalize all the concentrations of the produced weight species $W^j$ to a common value.

\begin{proposition} \label{prop bound wj}
  There exist positive constants $C_1$ and $C_2$, independent of $\abs{I}, \abs{\J}, \abs{K}$, $\Tr$, and $M$, such that for every $j\in \J$ the concentration $w^j(T^0_1)$ satisfies
  \[
    \AAbs{w^j(T_1^0) - \frac{b_1}{b_2}} \leq C_1 \mathrm{e}^{-C_2\Tr }.
  \]
  Consequently, for $\Tr\to \infty$, we have $w^j(T_1^0) \to \frac{b_1}{b_2}.$
\end{proposition}

The explicit values of the constants are given in the proof of the proposition in Appendix \ref{sec:proof prop bound wj}. This result states that after long renormalization $\Tr$, all weight concentrations $w^j(T_1^0)$ are approximately equal to the shared value $\frac{b_1}{b_2}$, which depends solely on the reaction rates of the renormalization network.

Note that the deterministic framework of reaction rate equations provides for a simpler analysis of the selection phase than the stochastic framework of SNNs~\citep{jaffard2026chani}: here, we do not need to prove that the selected sets $j$ are the ones whose fluxes surpass a certain threshold; this is automatically the case. Moreover, as was discussed in Section \ref{sec analogy SNN}, information is encoded more efficiently in the CRN than in the SNN, because we start with an empty set of selected sets $j$. A weight species $W^j$ is created and the index $j$ is added to the system whenever the  flux $\Phi^j$ exceeds the threshold.

\subsection{Average learning} \label{sec av learn}
We next prove that the presented CRN correctly implements expert aggregation algorithm EWA, and that it solves the learning task on average.

For $k\in K$, we define the family $(\bw^{j\to k}_m)_{j\in \J, m\in [M+1]}$ by
\[
  \forall j\in \J, \quad \bw^{j\to k}_1 = \frac{b_1}{b_2\abs{\J}}
\]
and for $m\geq 1$
\[
  \forall j\in \J, \quad \bw^{j\to k}_{m+1} = \frac{b_1}{b_2} \frac{\exp\left(G^{j\to k}_m\right)}{\sum_{j'\in \J} \exp(G^{j'\to k}_m)}
\]
where $G^{j\to k}_m$ is defined in Eq.~\eqref{eq def cum gain}. Then, the probability distribution generated by the EWA algorithm with gains $(g^{j\to k}_m)_{j\in \J}$ is $(\frac{b_2}{b_1} \Bar{w}^{j\to k}_m)_{j\in \J, m\in [M]}$. In the following proposition, we prove that the weight concentrations $(w^{j\to k}(T^2_m))_{j\in \J, m\in [M]}$ approximate this distribution.

\begin{assumption}[Bounded flux] \label{assump delta}
  There exists a constant $\alpha>0$, independent of $M$, such that $\sup_{j\in \J, t\in [T^1_0,T^2_M]} \Phi^j(t) \leq \alpha$. Besides, the quantity $\delta := s_0 - 2\frac{b_1}{b_2}\alpha$ is positive.
\end{assumption}

\begin{proposition} \label{prop cvg EWA}
  Under Assumption \ref{assump delta} (bounded flux),
  there exist positive constants $C_3$, $C_4$, and $C_5$, independent of $\abs{I}, \abs{\J}, \abs{K}, \Tr$, and $M$,
  such that if
  $\Tr \geq C_3$
  then for every $j\in \J$, $k\in K$, and $m\in [M]$ we have
  \[
    \abs{w^{j\to k}(T^2_m) - \Bar{w}^{j\to k}_{m+1}} \leq C_4 \mathrm{e}^{-C_5\abs{\J} \Tr}.
  \]
  Therefore, for $\Tr\to \infty$, we have $w^{j\to k}(T^2_m) \to \Bar{w}^{j\to k}_{m+1}$.
\end{proposition}

The exact error term is given in the proof of the proposition in Appendix \ref{sec:proof prop cvg EWA}. This result shows that in each training round $m$, the EWA network \eqref{eq crn EWA} and the renormalization network \eqref{eq renorm wk} are effectively implementing the expert aggregation algorithm EWA. Note that for the approximation to hold, we only need the duration $\Tr$ of the renormalization network to be large.

\begin{assumption}[$\xi$-balanced] \label{assump xi}
  There exists a constant $\xi>0$, independent of $M$, such that for every $k\in K$, $M^k/M \geq \xi$.
\end{assumption}

This assumption requires that the fraction of training samples $o_m$ of any class is at least $\xi$. This ensures that the CRN is exposed to a significant number of samples from each class during training. In the following proposition, we derive a regret bound on the species discrepancy of $X^k$ under this assumption.

\begin{proposition}[Regret bound of species $X^k$] \label{prop reg}
  Let $k\in K$. Suppose Assumptions \ref{assump delta} (bounded flux) and \ref{assump xi} ($\xi$-balanced) hold, and let $\eta = \frac{\xi}{2\Tl s_0} \sqrt{8\frac{\ln(\abs{\J})}{M}}$. Then, there exist positive constants $C_6$, $C_7$, and $C_8$, independent of $\abs{I}, \abs{\J}, \abs{K}$, $\Tr$, and $M$, such that if $\Tr \geq  C_6$ we have
  \[
    \disc^k \geq \max_{q \in \frac{b_1}{b_2}\mathcal{P}_I}  \disc^k(q^k) -C_7\left(\abs{\J} \mathrm{e}^{-C_8\Tr} + \ln(\abs{\J})^{1/2} M^{-1/2}\right).
  \]
\end{proposition}

The exact error term with explicit constants is provided in the proof of the proposition in Appendix \ref{sec:proof prop reg}. The specific choice of $\eta$ minimizes the regret bound of the expert aggregation algorithm EWA (see Appendix \ref{appendix EWA} for details).
This proposition states that the species discrepancy of every species $X^k$ exceeds the maximally achievable discrepancy under constant weights, up to an error term with two components. The first component, of order $O(\mathrm{e}^{-C_2 \Tr})$, arises from the approximation error to the true expert aggregation algorithm as derived in proposition \ref{prop cvg EWA}, and from the distance to equilibrium in the output species decay \eqref{eq decay xk}. The second component, of order $O(M^{-1/2})$, comes from the regret bound of the expert aggregation algorithm EWA itself, as detailed in Appendix \ref{appendix EWA}. Consequently, for $\Tr \to \infty$ and $M \to \infty$, the total error tends to zero.

The above results establish that for sufficiently long renormalization duration $\Tr$ and sufficiently many training samples $M$, the species discrepancies  $X^k$ are close to their optimal values. This guarantees that, on average, the concentration of species $X^k$ is higher after the network has been presented with a sample of class $k$. This result, however, is local to the behavior of species $X^k$ during learning. The following theorem provides the global result for the whole network.

\begin{theorem}[Oracle inequality] \label{th oracle}
  Suppose Assumptions \ref{assump delta} (bounded flux) and \ref{assump xi} ($\xi$-balanced) hold, let $\eta = \frac{\xi}{2\Tl s_0} \sqrt{8\frac{\ln(\abs{\J})}{M}}$ and $C_6, C_7$, and $C_8$ be the constants defined in Proposition \ref{prop reg}. Then, if $\Tr \geq  C_6$ we have
  \[
    \Disc \geq \max_{q^K \in (\frac{b_1}{b_2}\mathcal{P}_{\J})^{\abs{K}}}  \Disc(q^K) -C_7\left( \abs{\J} \mathrm{e}^{-C_8\Tr} +  \ln(\abs{\J})^{1/2} M^{-1/2}\right).
  \]
\end{theorem}

The error term is identical to that in Proposition \ref{prop reg}. The theorem is proven in Appendix \ref{sec:proof th oracle}. It asserts that the network discrepancy of the CRN exceeds the maximally achievable discrepancy under constant weights, up to this error term. The family $q^K = (q^k)_{k\in K}$ achieving this maximum acts as an oracle with global knowledge of the behavior of each output species $X^k$, while the CRN itself is designed so that species associated with different classes do not interact. Consequently, the inequality implies that the average classification performance of our CRN—as quantified by the network discrepancy—is, up to an error term that vanishes as $\Tr \to \infty$ and $M \to \infty$, comparable to that of the oracle.

This theorem can be related to Corollary $12$ of \cite{jaffard2026chani}, which established an analogous result for a SNN. However, the assumptions on the SNN were more restrictive and included  Assumption \ref{assump natures and nb obj} (repetitive training samples). The results for the CRN are less restrictive due to the simpler structure of the selection phase and the fact that input species directly influence the output species through the forward-pass reactions \eqref{eq net}. Achieving the same in a SNN requires the treatment of multiple hidden layers (see Figure \ref{fig:CRN SNN}).

Consequently, the above theorem constitutes a first step toward demonstrating that CRNs can learn in a manner comparable to networks of biological neurons, while relying on a simpler structural framework. Moreover, the proof of the theorem allows us to identify the role each component of the CRN plays (see Section \ref{sec discussion}).
An important remark is that all species associated with a given class $k$ (\ie the species $X^k$, $(H^{j \to k})_{j\in \J}$, and $(W^{j \to k})_{j\in \J}$) evolve independently of species associated with other classes. Proposition \ref{prop reg} is obtained by considering only the dynamics corresponding to a single class, yielding a local regret bound for the species $X^k$. The central contribution of Theorem \ref{th oracle} is to derive a global result for the entire network by aggregating these local bounds. This points at a theoretical understanding of emergence in learning networks, whereby complex, global behavior arises from the interaction of simple components.

Nevertheless, although the proposed CRN is shown to asymptotically attain performance comparable to that of the oracle, this does not guarantee that the oracle itself achieves satisfactory performance. In particular, if the input species $X^i$ do not encode information that is relevant to the classification task, the oracle's performance will necessarily be poor. For example, in the context of a gene regulatory network designed to discriminate among different cell states, if the set of input species $I$ consists of transcription factors entirely unrelated to those states, the learning problem becomes unsolvable regardless of the values of the weight concentrations. In the following subsection, we therefore analyze the asymptotic network performance under more restrictive assumptions and provide explicit conditions under which the classification task is solvable by the CRN.

\subsection{Asymptotic analysis} \label{sec asymp analysis}

Under Assumption \ref{assump natures and nb obj} (repetitive training samples), we can extend the previous analysis by computing explicit asymptotic weight concentrations. Note that Assumption \ref{assump natures and nb obj} implies Assumption \ref{assump xi} ($\xi$-balanced) with $\xi =\frac{1}{\abs{\Obj}}$.

\begin{theorem}[Asymptotic output weight concentrations] \label{theo lim wk}
  Suppose Assumptions \ref{assump natures and nb obj} (repetitive training samples) and \ref{assump delta} (bounded flux) hold, and let $C_6$ be the constant defined in Proposition \ref{prop reg}. For $k\in K$, let $\J^k := \arg\max_{j\in \J} \FDisc^{j\to k}$ and define the family $\bw^k\in \frac{b_1}{b_2}\mathcal{P}_{\J}$ by
  \[
    \forall j\in \J, \quad \bw^{j\to k} := \frac{b_1}{b_2\abs{\J^k}} \mathbb{1}_{j\in \J^k}\, .
  \]
  Then, there exist positive constants $C_9,C_{10}$ and $C_{11}$, independent of $\abs{I}, \abs{\J}, \abs{K}$, $\Tr$, and $M$, such if $\Tr\geq C_6$ then for every $k\in K$, at the end of the learning phase the weight concentrations satisfy
  \[
    \Norm{w^k(T^2_M) - \bw^k}_2 \leq C_9 \left( \left(\abs{\J}\ln(\abs{\J})M\right)^{1/2}  \mathrm{e}^{-C_{10} \Tr} + \abs{\J}^{3/2} \mathrm{e}^{-C_{11} \ln(\abs{\J})^{1/2} M^{1/2}}\right).
  \]
  Therefore, for $M\to \infty$ and $\frac{\Tr}{\ln(M)}\to \infty$, we have $w^k(T^2_M) \to \bw^k$.
\end{theorem}

The exact form of the error term with explicit constants is given in the proof of the theorem in Appendix \ref{sec:proof theo lim wk}. This theorem provides the asymptotic behavior of the output weight concentrations at the end of the learning phase. In the limit, the concentration weight $\bw^{j \to k}$ is nonzero only for those sets $j$ that achieve the maximal flux discrepancy $\FDisc^{j \to k}$, and the concentrations are identical for all such maximizing sets.
Consequently, in the asymptotic regime, the forward-pass equations \eqref{eq net} imply that the species $X^k$ is produced exclusively through reactions involving sets $j$ of input species that exhibit higher-than-average flux when samples of class $k$ are presented. During inference, the weight concentrations are fixed to their values at the end of the learning phase (see Section \ref{sec inference}). The above result therefore characterizes the network’s inference behavior for large $M$ and $\Tr/\ln(M)$.

The error term has two components. The first part, in $O(M^{1/2} \mathrm{e}^{-O(\Tr)})$, comes from approximations in the two renormalization networks. The second part, in $O(\exp(-O(M^{1/2}))$, comes from the specific expert-aggregation algorithm, EWA.

  This theorem can be related to Theorem $14$ of \cite{jaffard2026chani}, which establishes a similar asymptotic analysis for the output synaptic weights of a SNN. However, the present framework is more general, and the above error bounds are tighter, owing to the simplifications enabled by the use of a CRN.

  Even if the limit output-weight concentration family $(\bw^k)_{k\in K}$ has a clear and relevant interpretation from a machine-learning point of view, one cannot know for sure whether it can solve the task, since the set of input species could code for features that are not relevant for the classification task. To finally conclude on the network's asymptotic performance, the more refined analysis in the following subsection closes this gap.

  \subsubsection{When classes are defined by feature correlations}

  We start by stating some additional assumptions that we will need.

  \begin{assumption}[Binary flux] \label{assump bin flux}
    There exist a threshold $\theta>0$ and a constant $p>0$ such that for all $j\in \J$ and $o\in \Obj$, $\int_o \Phi^j \in \{0,p\}$. If $\int_o \Phi^j = p$ we say that item $o$ has features $j$. Besides, for $j\in \J$, let $\Obj^j:= \{o\in \Obj, \int_o \Phi^i = p\}$ the set of item types with features $j$. Then, all $\Obj^j$ have the same cardinality.
  \end{assumption}

  The first part of this assumption is fulfilled if, for instance, the concentrations $x^i$ are either zero or equal to a common function $x$, such that each flux is either $0$ or $\int_o x^n$. The second part is a technical assumption that we make to facilitate the derivations.

  \begin{assumption}[Class decomposition] \label{assump class dec}
    The threshold $\theta$ in Assumption \ref{assump bin flux} (binary flux) is such that for every $k\in K$, there exists a set $E^k\subset \J$ such that $k = \bigcup_{j\in E^k} \Obj^j$.
  \end{assumption}

  This assumption implies that the classes to which samples are assigned are determined by correlations between features of the set $I$, which are encoded in the set $\J$. More precisely, a sample type $o$ belongs to class $k$ if there exists a subset $j\in E^k$ of features such that $o$ has all features belonging to this subset.
  Consequently, the complexity of these correlations increases with the cardinality $n$ of the subsets $j$: as the depth of the CRN increases, each set $\Obj^j$ represents progressively more complex patterns among the samples.

  \begin{theorem} \label{th equiv}
    Suppose Assumption \ref{assump bin flux} (binary flux) holds. Then the following statements are equivalent:

    $(i)$ Assumption \ref{assump class dec} (class decomposition) is satisfied.

    $(ii)$ There exists an optimal weight family.

    $(iii)$ The limit weight family defined in Theorem \ref{theo lim wk} is an optimal weight family.
  \end{theorem}

  This theorem is proven in Appendix \ref{sec:proof th equiv}. It establishes that, provided the classes can be expressed as combinations of correlations among $n$ features encoded by the input species, the CRN learns to correctly classify all sample types. Furthermore, one or several optimal families of weights exists, and the CRN's output weight concentrations converge to one of them. This learning guarantee is analogous to the Perceptron Learning Theorem \citep{rosenblatt1962principles}, which asserts that if there exists a set of weights capable of classifying the data---equivalently, if the data are linearly separable---then the perceptron learning algorithm will converge to such a solution.

  This result can be related to Theorem $19$ of \cite{jaffard2026chani}. However, Assumption \ref{assump class dec} (class decomposition) applies in a broader range of settings than its SNN counterpart, as it does not restrict correlations to combinations of features with cardinalities that are powers of two. Here, any combination is admissible.

  If Assumption \ref{assump class dec} is not satisfied, the depth $n$ of the network can be increased, which increases the complexity of the feature correlations encoded by the sets $j\in \J$. Then, the assumption will eventually be satisfied and the CRN will be able to solve the task. The is reminiscent of the infinite-depth limit for ANNs \citep{hayou2023infinitedepthlimitfinitewidthneural}.

  Under Assumption \ref{assump bin flux} (binary flux), we can also compute the Vapnik - Chervonenkis dimension of the CRN. For a given classification algorithm, the VC-dimension is defined as the size of the largest set of points that it can shatter. This quantity provides an upper bound on the complexity of the class of functions that can be learned. It notably appears in bounds on the generalization error of machine-learning algorithms (see for instance \citet[Corollary $3.19$]{mohri2018foundations}).

  We model the samples to be classified as bit strings $o=(o_i)_{i\in I}$, where $o_i=1$ if $o$ has feature $i$ and $o_i = 0$ otherwise. These strings belong to a space of dimension $\abs{I}$. A set $j\in \J$ of input species has a positive flux integral $\int_o\Phi^j$ if $o$ has all features of the set $j$. In this case, we say that $j$ is activated by $o$.

  We consider a binary-output classifier and define the hypothesis set $\h:={\mathbb{1}_{F} \text{ such that } F\subset \J}$, where for a sample $o$, $\mathbb{1}_{F}(o) = 1$ if and only if there exists $j\in F$ activated by $o$, and $0$ otherwise. This hypothesis set corresponds to the class of functions that the CRN can learn. We then obtain the following result as proven in Appendix \ref{sec:proof prop vcdim}.

  \begin{proposition} \label{prop vcdim}
    The VC-dimension of the hypothesis set $\h$ is the size of the selected input species subsets,
    \[\operatorname{VCdim}(\h) = \abs{\J}\, .\]
  \end{proposition}

  The VC-dimension of the CRN, which quantifies the complexity of the classes that the CRN can learn, is equal to the size of the set of selected input species subsets $\abs{\J}$. However, some error terms in the approximation results above grow with $\abs{\J}$, and the number of reactions and species in the CRN grows linearly with $\abs{\J}$ (see Section \ref{sec net complexity}). This reflects the typical trade-off in supervised learning algorithms between expressiveness and efficiency.

  \subsection{The roles of different network components} \label{sec discussion}

  We use the theoretical results above to introspect the role of each component of the proposed CRN. At the core of the network are the forward-pass reactions~\eqref{eq net}. They describe how input species and weight species, acting as catalysts, induce the production of output species. The rate equation of species $X^k$ (see equation \eqref{eq rate xk}) resembles the computations performed across multiple layers in a SNN as shown in Figure \ref{fig:CRN SNN}. The mass-action rate equations of this network part (see Appendix \ref{appendix rate eq}) are key to the regret bounds for the expert aggregation algorithm (see Appendix \ref{appendix EWA}), from which Theorem \ref{th oracle} is derived. They are also essential for analyzing the network behavior in the regime of asymptotic weights (Theorem \ref{th equiv}).

  The second ingredient is the EWA network \eqref{eq crn EWA}, in which gain-function species $H^{j\to k}$ are produced with the input species acting as catalysts. Together with the renormalization network \eqref{eq renorm wk}, the gain-function species modify the weight concentrations $w^{j\to k}$, feeding back into the forward-pass reactions \eqref{eq net}. This makes it possible to compute the asymptotic weight concentrations (Theorem \ref{theo lim wk}), admitting a meaningful interpretation of the learning task. This asymptotic analysis enables the derivation of Theorem \ref{th oracle}. However, alternative reactions could be considered in place of the EWA network \eqref{eq crn EWA}: in principle, reactions implementing {\em any} expert-aggregation algorithm for which a regret bound can be established could be used to derive Theorem \ref{th oracle}, and it is likely that other algorithms also yield meaningful asymptotic behavior.

  As discussed in Section \ref{sec av learn}, a remarkable fact about the reactions of the learning phase is that species related to different classes $k$ do not directly interact with each other. Yet, we obtain global results about the network behavior. This provides a potential explanation for how complex, system-level behavior emerges from interactions among simple cellular components.

  The remaining reactions, particularly in the selection phase, play a less central role in the analysis. The selection network \eqref{eq Wj} is mainly required for complexity control. Without selection (\ie with a threshold $\theta = 0$), the set $\J$ would have cardinality $\binom{\abs{I}}{n}$, leading to a prohibitively large number of species and reactions (see Section \ref{sec net complexity} for details on the network complexity).
  The two renormalization networks \eqref{eq renorm wj} and \eqref{eq renorm wk} enforce mass conservation. This guarantees a fixed value for the concentrations $w^j$ and a fixed sum $\sum_{j\in \J} w^{j\to k}$ for each $k\in K$ at equilibrium, such that the weights are a partition of unity. This conservation constraint is required for expert aggregation algorithms, which necessarily provide probability distributions.

  \section{Numerical results} \label{sec num res}

  We demonstrate the function of the presented CRN in a simple benchmark classification task.
  Specifically, we train the CRN on the handwritten digits data set provided by {\tt scikit-learn}. This data set consists of 1797 digital images of 64 (8$\times$8) pixels each (see Figure \ref{fig:digits} for examples), which we randomly divide into a training set (80\% of the images, \ie $1437$) and a test set (20\% of the images, \ie $360$). Among the training set, $40$ images are used for the selection phase and $1397$ for the learning phase. The set of features $I$ corresponds to the pixels: each pixel $i\in I$ is represented by an input species $X^i$ with constant concentration during the presentation of the $m^\text{th}$ image, $x^i = ( y^i_m + \xi^i_m)_+$.
  Here, $y^i_m\in [0,1]$ is the gray level of pixel $i$ in the $m^\text{th}$ image, and $\xi^i_m$ is a Gaussian variable with mean $0$ and variance $\sigma^2>0$. The pixel-wise independent Gaussian noise is added to the raw pixel intensities in order to simulate degraded data quality. The positive part is then taken to ensure that all intensities are non-negative.

  Since we evaluate the network accuracy on the test set, we use the simplified network architecture described in the last paragraph of Section \ref{sec net complexity}. This means that the output weights are renormalized only once, after learning, and the output species are produced exclusively during the testing phase.

  \begin{figure}
    \centering
    \includegraphics[width=0.8\linewidth]{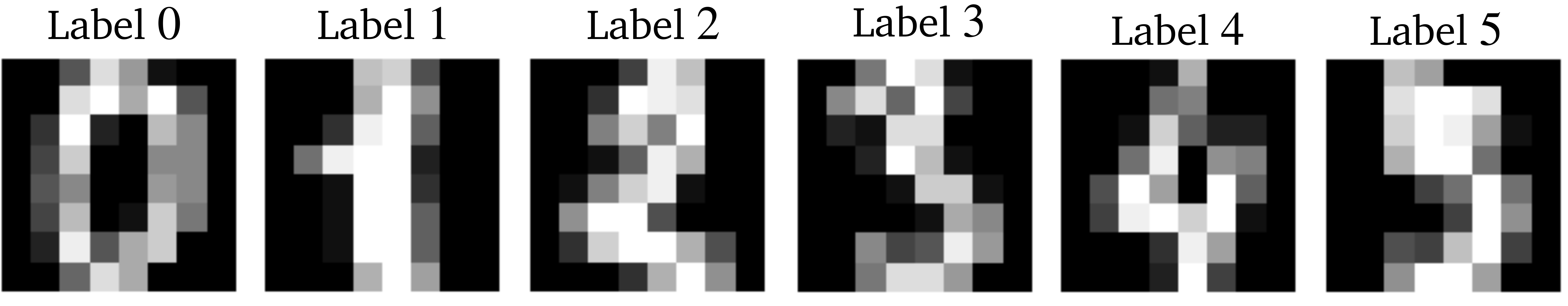}
    \caption{Example images from the handwritten digits data set. Each image shows a digit (ground-truth labels above the images) in 8$\times$8 grayscale pixels.}
    \label{fig:digits}
  \end{figure}

  We quantify performance as a function of network complexity by controlling network complexity as follows: during the selection phase, instead of fixing a threshold $\theta$ and selecting the fluxes exceeding it, we fix the number of selected fluxes $\abs{\J}$ and retain the $\abs{\J}$ largest ones. The two approaches are equivalent, as fixing $\abs{\J}$ implicitly defines a threshold $\theta$ that separates the fluxes in the same way.

  \begin{figure}
    \centering
    \includegraphics[width=1\linewidth]{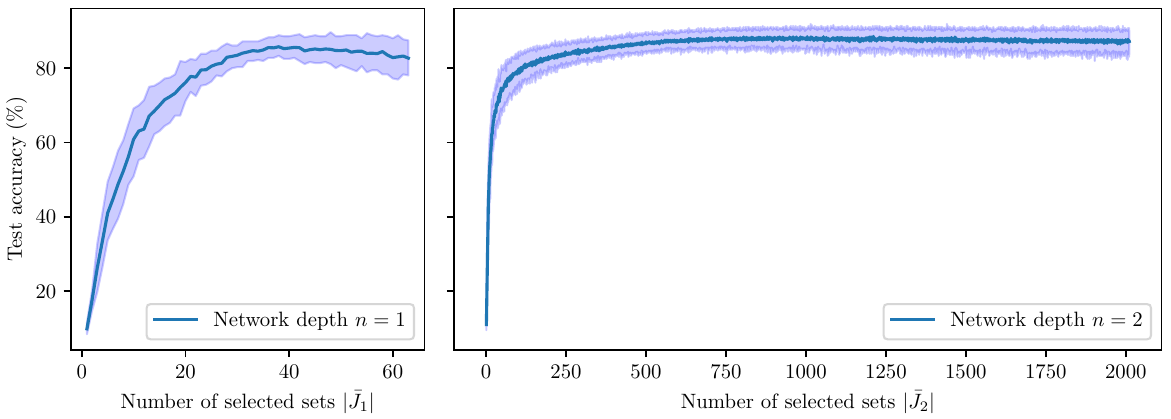}
    \caption{Numerical results on the handwritten digits data set. We plot the CRN's classification accuracy on the test set as a function of network complexity, quantified by the number of selected input sets $\abs{\bar{J}_n}$ during the selection phase. We show results for network depths $n=1$ (no hidden layer, left panel) and $n=2$ (one hidden layer, right panel). The shaded bands are the 10\% confidence intervals over 100 independent repetitions of each configuration. All rate constants are equal to 1, except $c_k$ and $d_k$ for $k\in K$, which are equal to $\frac{10}{9}\eta$ and $10\eta$, respectively. The initial species concentrations are $A=1$ and $S=3$,
    and the noise variance is set to $\sigma^2=0.00001$. For network depth $n=1$, we chose $\eta = 0.0005$ and for network depth $n=2$, we chose $\eta = 0.0001$. }
    \label{fig:accuracy}
  \end{figure}

  The results for network depths $n=1$ and $n=2$ are presented in Figure \ref{fig:accuracy}. The curves represent the classification accuracy on the test set as a function of the network complexity $\abs{\J}$. The shaded bands show the 10\% confidence intervals over 100 independent repetitions of each configuration. The sources of randomness are the Gaussian noise introduced in the encoding of the input species and the random split between the training and test sets, independently for each network complexity.

  For $n=1$, the output species receive information from individual input species, each encoding a single pixel. Since there are 64 input species, $\abs{\bar{J}_1}$ ranges from 1 to 64. The maximum accuracy, $85.8\%$, is achieved for $\abs{\bar{J}_1} = 37$ selected sets of single-input species. Beyond this value, the accuracy decreases slightly.

  For $n=2$, each output species receives information from a pair of input species, corresponding to a pair of pixels. There are 2016 unique pairs, so $\abs{\bar{J}_2}$ ranges from 1 to 2016. The accuracy reaches its maximal value of $88.6\%$ for $\abs{\bar{J}_2} = 1210$ selected pairs of input species. To match the maximum accuracy of the CRN with $n=1$, the CRN with $n=2$ requires $\abs{\bar{J}_2} = 382$, which is a higher complexity than the best CRN with $n=1$.

  These results indicate that the network with depth $n=1$ offers the better cost-performance ratio, but increasing the depth achieves higher absolute performance. The case $n=1$ also reveals that, on average, only 37 of the pixels provide relevant information.
  These observations differ substantially from the numerical experiments reported by \cite{jaffard2026chani}, Section 5.2, where a SNN was trained on the same task. There, without hidden layers (corresponding to $n=1$), the reported test accuracy was only 53\%. With one hidden layer (corresponding to $n=2$), the accuracy increased with the number of selected neurons until saturating at 83.5\%. Thus, the CRN with $n=1$ performs significantly better than a SNN without hidden layer, and even outperforms the SNN with one hidden layer. The performance of the CRN with $n=2$ also surpasses the SNN with one hidden layer. In summary, the simplest CRN ($n=1$ with slightly more than half of the pixels encoded) outperforms both SNN configurations from \cite{jaffard2026chani}, at substantially lower complexity. This is consistent with our theoretical results, which show that a CRN without hidden layers can achieve similar learning guarantees as a SNN with hidden layers.

  \section{Conclusion}

  We introduced a chemical reaction network without hidden layers that can provably learn a classification task more efficiently than a spiking neural network with hidden layers in a similar problem setting. We derived local regret bounds that characterize the learning dynamics of the output chemical species encoding the classes, arising from the implementation of an expert-aggregation algorithm within the CRN. We then established an oracle inequality describing the global behavior of the network, demonstrating that it solves the classification task on average under non-restrictive assumptions.

  In the more restricted setting where training samples are presented repeatedly, corresponding to epochs in the training of artificial neural networks,
  we further analyzed the asymptotic behavior of the CRN and characterized the classes of functions it is capable of learning. If the input features are actually relevant to the task, we explicitly gave the  Vapnik--Chervonenkis (VC) dimension of the network. Finally, we implemented the proposed network and evaluated its performance on the classification of handwritten digits. To the best of our knowledge, the CRN presented here is the first CRN for which learning behavior has been mathematically proven.

  The architecture of the proposed CRN is inspired by a spiking neural network (SNN) for which analogous theoretical guarantees have previously been derived~\citep{jaffard2026chani}. By establishing comparable results in a CRN, we demonstrated that CRNs can encode and process information just as neuronal networks can. What's more, our analysis indicated that CRNs may be more powerful learning machines than SNNs. Indeed, the presented CRN achieves the same learning guarantees as a SNN under weaker assumptions and for lower network complexity. This was confirmed in our numerical experiments. It suggests that CRNs without hidden layers can solve learning tasks for which SNNs require hidden layers. This might be due to the fact that multiplication is intrinsic in the physics of mass-action kinetics, whereas it needs to be approximated by sequences of nonlinearly weighted summations in a neural network.

  Future extensions of this work could generalize the results to a broader class of CRNs. This could then enable finding actual biochemical examples of learning networks in biological cells by screening for the defining network features of that CRN class. The discussion of the role of each reaction module in Section \ref{sec discussion} provides an initial step in this direction. It could eventually reveal structural or functional similarities with certain biochemical processes and enable their reinterpretation as learning machines within cells.

  \subsection*{Supplementary information}
  The code used to produce the numerical results can be found at:

  \noindent \href{https://github.com/SophieJaffard/CRN}{https://github.com/SophieJaffard/CRN}.

  \subsection*{Acknowledgments}

  S.J. was financially supported by a postdoctoral fellowship from the ELBE Postdoc Program of the Center for Systems Biology Dresden.

  \subsection*{Statements and declarations}
  The authors have no competing interests to declare that are relevant to the content of this article.

  \newpage

  \section{Appendix}

  \subsection{Details on the regret bound of EWA} \label{appendix EWA}

  The regret bound $R_M$ for the EWA algorithm was given by \citet{cesa2006prediction} 
  for losses (\ie negative gains) in $[0,1]$. It was later generalized by \citet{stoltz2010agregation} to losses in the interval $[a,b]$ for any $a<b\in \mathbb{R}$:
  \[R_M \leq \frac{\ln(\abs{E})}{\eta}+ \eta \frac{(b-a)^2}{8}M \, ,  \]
  where $\abs{E}$ is the number of experts.
  With $\eta = \frac{1}{b-a}\sqrt{8\ln(\abs{E})/M}$, we obtain
  \begin{equation} \label{eq regret EWA}
    R_M \leq (b-a)\sqrt{\frac{M}{2}\ln(\abs{E})} \, .
  \end{equation}
  This choice of $\eta$ assumes that we know the time horizon $M$ in advance. If this is not the case, we can use a time-dependent learning rate $\eta_m = \frac{1}{b-a}\sqrt{8\ln(\abs{I})/m}$, which gives the bound
  \[R_M \leq \sqrt{2M\ln(\abs{E})} + \sqrt{\frac{\ln(\abs{E})}{8}}\]
  of the same order of magnitude.

  \subsection{Proofs}\label{sec:proofs}

  In the following proofs, we denote by $a_0 >0$ the initial concentration of species $A$ and by $h_0 >0$ the initial concentration of species $H^{j\to k}$. Note that the concentration of species $A$ does not evolve and stays equal to $a_0$. For  $j\in \J$,  we write $x^j := w^j\Phi^j$ and $\Bar{x}^{j} := \frac{b_1}{b_2}\Phi^j $.
  We start by stating the rate equations of the species in the CRN.

  \subsubsection{Rate equations} \label{appendix rate eq}

  We recall that we start from $x^k(T^0_1)=0$ for all $k$. During the time interval $[T^0_m, T^1_m]$, the concentrations of the species $W^{j\to k}$ stay constant. The concentrations $x^k$ and
  $h^{j\to k}$ for $j\in \J$ and $k\in K$ are governed by the following rate equations:
  \begin{align}
    \forall k\in K, \quad & \frac{dx^k}{dt} = a_1\sum_{j\in \J} x^j(t)w^{j\to k}(T^0_m) \label{eq rate xk}\\
    \forall j\in \J, \quad & \frac{dh^{j\to k^\star}}{dt} = d_{k^\star}( x^j(t) + s_0) h^{j\to k^\star}(t) \\
    \forall j\in \J, k\neq k^\star, \quad & \frac{dh^{j \to k}}{dt} = c_k (s_0 - x^j(t)) h^{j\to k^\star}(t) \label{eq h}
  \end{align}

  Therefore, for $t\in [T^0_m,T^1_m]$,
  \begin{align}
    \forall k\in K, \quad & x^k(t) = x^k(T^0_m) + a_1\int_{T^0_m}^{t}  \sum_{j\in \J} x^j(u)w^{j\to k}(T^0_m) du \label{eq xk} \\
    \forall j\in \J, \quad &  h^{j\to k^\star}(t) = h^{j\to k^\star}(T^0_m) \exp\left(d_{k^\star}\int_{T^0_m}^t (x^j(u) + s_0)\, \mathrm{d}u \right) \label{eq h1} \\
    \forall j\in \J, k\neq k^\star, \quad & h^{j\to k}(t) = h^{j\to k}(T^0_m) \exp\left(c_{k}\int_{T^0_m}^t (s_0 - x^j(u))\, \mathrm{d}u \right) . \label{eq h2}
  \end{align}

  During the time interval $[T^1_m, T^2_m]$, the concentrations $h^{j\to k}$ stay constant, and the concentrations $x^k$ and $w^{j\to k}$ for $j\in \J$ and $k\in K$ are governed by the rate equations:
  \begin{align*}
    \forall k\in K, \quad & \frac{dx^k}{dt} = - a_2x^k(t) \\
    \forall j\in \J, k\in K, \quad & \frac{dw^{j\to k}}{dt} = b_1 h^{j\to k}(T^1_m) - b_2 w^{j\to k}(t) \sum_{j'\in \J} h^{j'\to k}(T^1_m)\, .
  \end{align*}
  Let $\Bar{w}^{j\to k}_{m+1} := \frac{b_1}{b_2}\frac{h^{j\to k}(T^1_m)}{\sum_{j'\in I}h^{j'\to k}(T^1_m)} $. Then, according to Lemma \ref{lemma sol renorm net}, we have
  \begin{align}
    \forall k\in K, \quad & x^k(t) = x^k(T^1_m)\mathrm{e}^{- a_2(t - T^1_m)} \label{eq xk 2} \\
    \forall j\in \J, k\in K, \quad & w^{j\to k}(t)= \Bar{w}^{j\to k}_{m+1} + (w^{j\to k}(T^1_m) - \Bar{w}^{j\to k}_{m+1})\mathrm{e}^{-b_2 (t - T^1_m) \sum_{j'\in I}h^{j'\to k}(T^1_m) }. \label{eq ode w}
  \end{align}

  \subsubsection{Preliminary propositions}

  As proven in \cite[Proposition 4]{gao2018properties}, we have:

  \begin{proposition} \label{prop iaf}
    Let $d\in \mathbb{N}$, $\eta\in \mathbb{R}$. Then, the softmax function $f$ with temperature $\eta$ defined on $\mathbb{R}^d$ by $x \mapsto \softmax(\eta x)$ is $\eta$-Lipschitz.
  \end{proposition}

  As proven in \cite[Proposition 26]{jaffard2026chani}, we have

  \begin{proposition} \label{prop prelim EWA}
    Let $d\in \mathbb{N}$, $A^1,\dots, A^d\in \mathbb{R}$. Let $E$ be the subset of $[d]$ defined by $E:= \arg\max_i A^i$, and let
    $w_i(M) := \frac{\exp(\sqrt{M}A^i)}{\sum_{l=1}^d \exp(\sqrt{M}A^l)}$. Let $A=\max_i A^i$.
    \begin{itemize}
      \item If $E=[d]$ then for every $i\in [d]$, $w_i(M) = \frac{1}{d}.$
      \item If $E\subsetneq [d]$, let $\Delta := \max_i A^i - \max_{i\notin E} A^i$. Then, for all $i$,
        \[\Abs{w_i(M) - \frac{1}{\abs{E}}\mathbb{1}_{i\in E}} \leq \frac{1}{\abs{E}}\max\Big(1,\frac{d-\abs{E}}{\abs{E}}\Big) \exp(-\sqrt{M}\Delta).\]
    \end{itemize}
  \end{proposition}

  \begin{lemma} \label{lemma sol renorm net}
    For $n\in \N$, consider species $(A_l)_{1\leq l \leq n}$ and $(B_l)_{1\leq l \leq n}$ whose concentrations evolve on $\R_+$ according to the chemical reaction equations
    \begin{align*}
      \forall l\in [n] \quad & A^l \xlongrightarrow{c_1}  A^l + B^l \\
      \forall l,l'\in [n], \quad &  A^{l'} + B^l \xlongrightarrow{c_2} A^{l'} .
    \end{align*}
    Then, the concentrations $(a^l)_{l\in [n]}$ stay constant and the concentrations $(b^l)_{l\in [n]}$ satisfy
    for all $l\in [n]$ and $t\geq 0$:
    \[
      b^l(t) = \frac{c_1}{c_2}\frac{a^l(0)}{\sum_{l'=1}^n a^{l'}(0)} + \left(b^l(0) - \frac{c_1}{c_2}\frac{a^l(0)}{\sum_{l'=1}^n a^{l'}(0)}\right)  \exp\left(-c_2 t\sum_{l'=1}^n a^{l'}(0)\right).
    \]
  \end{lemma}

  \begin{proof}
    For every $l\in [n]$, the concentration $a^l$ stays constant, whereas the concentration $b^l$ follows the rate equation
    \[
      \frac{db^l}{dt} = c_1 a^l(0) - c_2 b^l \sum_{l'=1}^n a^{l'}(0)\, .
    \]
    A particular solution to this equation is $b^l_P : t \mapsto \frac{c_1}{c_2}\frac{a^l(0)}{\sum_{l'=1}^n a^{l'}(0)}$, such that the difference $b^l - b^l_P$ solves the homogeneous equation
    \[
      \frac{d(b^l - b^l_P)}{dt} = - c_2 (b^l - b^l_P) \sum_{l'=1}^n a^{l'}(0)\, .
    \]
    Therefore, we have
    \[
      (b^l-b^l_P)(t) = (b^l - b^l_P)(0) \exp\left(-c_2 t\sum_{l'=1}^n a^{l'}(0)\right) ,
    \]
    \ie
    \[
      b^l(t) = \frac{c_1}{c_2}\frac{a^l(0)}{\sum_{l'=1}^n a^{l'}(0)} + \left(b^l(0) - \frac{c_1}{c_2}\frac{a^l(0)}{\sum_{l'=1}^n a^{l'}(0)}\right)  \exp\left(-c_2 t\sum_{l'=1}^n a^{l'}(0)\right).
    \]
  \end{proof}

  \subsubsection{Proof of Proposition \ref{prop bound wj}}\label{sec:proof prop bound wj}
  According to Lemma \ref{lemma sol renorm net}, we have
  \[
    w^j(T^0_1) = \frac{b_1}{b_2} + \left( w^j(\Tsel) - \frac{b_1}{b_2}\right) \mathrm{e}^{-b_2 a_0 (T^0_1 - \Tsel) }.
  \]
  Furthermore, in $[0, \Tsel]$ the concentrations $w^j$ fulfill the rate equation $\frac{{d}w^j}{{d}t} =f(\Phi^j(t))$,
  so $w^j(\Tsel) = \int_0^{\Tsel}f(\Phi^j(t)) \,\mathrm{d}t$. Therefore, we have $
  0 \leq w^j(\Tsel) \leq \norm{f}_\infty \Tsel $
  and get
  \[
    \AAbs{w^j(T^0_1) - \frac{b_1}{b_2}} \leq  \max\left\{\frac{b_1}{b_2},\, \norm{f}_{\infty} \Tsel\right\} \mathrm{e}^{-b_2 a_0 \Tr }.
  \]

  \subsubsection{Proof of Proposition \ref{prop cvg EWA}}\label{sec:proof prop cvg EWA}

  We start by proving the following proposition, which ensures that the rate of equation \eqref{eq h} stays positive.

  \begin{proposition} \label{prop rate h pos}
    Suppose $\Tr \geq \ln\left(\frac{b_2\norm{f}_\infty \Tsel}{b_1}\right) (b_2a_0)^{-1}$ and Assumption \ref{assump delta} (bounded flux) holds. Then, for all $t\in [T^1_0,T^2_M]$ and $j\in \J$, $s_0 - x^j(t) \geq \delta$.
  \end{proposition}
  \begin{proof}
    According to the proof of Proposition \ref{prop bound wj}, we have
    \[
      w^j(T^1_0) \leq \frac{b_1}{b_2} + \max\left\{\frac{b_1}{b_2} \mathrm{e}^{-b_2 a_0 \Tr}, \norm{f}_\infty \Tsel  \mathrm{e}^{-b_2 a_0 \Tr} \right\}.
    \]
    Besides, $\frac{b_1}{b_2} \mathrm{e}^{-b_2 a_0 \Tr} \leq \frac{b_1}{b_2}$, and the condition on $\Tr$ implies that
    \[\norm{f}_\infty \Tsel  \mathrm{e}^{-b_2 a_0 \Tr} \leq \frac{b_1}{b_2}\]
    as well, so we have $w^j(T^1_0) \leq 2\frac{b_1}{b_2}$. Thus, according to Assumption \ref{assump delta}, for $t\in [T^1_0, T^2_M]$ and $j\in \J$, since the concentration $w^j$ does not evolve during learning, we have
    \begin{align*}
      x^j(t) &= w^j(T^1_0)\Phi^j(t) \leq 2\frac{b_1}{b_2} \alpha
    \end{align*}
    and therefore $s_0 - x^j(t) \geq \delta$.
  \end{proof}

  With this result, we can now prove Proposition \ref{prop cvg EWA}. Let
  \[C_3 := \max\left\{\ln\left(\frac{b_2}{b_1}\norm{f}_\infty \Tsel\right)(b_2a(0))^{-1},\, \ln(2) (b_2 h_0)^{-1}\right\}.\]
  Let $j\in \J$, $k\in K$, and $m\in [M]$. According to Eq.~\eqref{eq ode w}, we have
  \[w^{j\to k}(T^2_m)= \Bar{w}^{j\to k}_{m+1} + (w^{j\to k}(T^1_m) - \Bar{w}^{j\to k}_{m+1})\mathrm{e}^{-b_2 \Tr \sum_{j'\in \J}h^{j'\to k}(T^1_m) }\, .\]
  In addition, according to Eqs.~\eqref{eq h1} and \eqref{eq h2}, since $d_k$ and $c_k$ are lower-bounded by $\frac{\eta}{\abs{K}-1}$ and since $s_0 - x^j$ is lower-bounded by $\delta$ on $[T^0_m, T^1_m]$, for $j'\in \J$ we have
  \begin{align*}
    h^{j'\to k}(T^1_m) \geq h_0 \exp\left(\frac{\eta}{\abs{K}-1} m \delta \Tl\right).
  \end{align*}
  Therefore, we have
  \begin{equation*}
    \mathrm{e}^{-b_2 \Tr \sum_{j'\in \J}h^{j'\to k}(T^1_m) } \leq  \mathrm{e}^{-b_2 \Tr \abs{\J} h_0 \exp\left(\frac{\eta}{\abs{K}-1} m \delta \Tl\right)} ,
  \end{equation*}
  which implies
  \begin{equation}  \label{eq min w}
    \abs{ w^{j\to k}(T^2_m) - \Bar{w}^{j\to k}_{m+1}} \leq  \abs{w^{j\to k}(T^1_m) - \Bar{w}^{j\to k}_{m+1}} \mathrm{e}^{-b_2 \Tr \abs{\J} h_0 \exp\left(\frac{\eta}{\abs{K}-1} m \delta \Tl\right)} .
  \end{equation}

  We prove the following equation by recursion on $m$:
  \begin{equation} \label{eq recursion}
    \forall m\in [M], \quad \abs{w^{j\to k}(T^2_m) - \Bar{w}^{j\to k}_{m+1}} \leq \min\left\{2\frac{b_1}{b_2},   4 \frac{b_1}{b_2} \mathrm{e}^{-b_2 \Tr \abs{\J} h_0 \exp\left(\frac{\eta}{\abs{K}-1} m \delta \Tl\right)}\right\}.
  \end{equation}

  \begin{itemize}
    \item Case $m=1$: According to Eq.~\eqref{eq min w}, we have
      \begin{align*}
        \abs{w^{j\to k}(T^2_1) - \Bar{w}^{j\to k}_2} &\leq  \abs{w^{j\to k}(T^1_1) - \Bar{w}^{j\to k}_2}\mathrm{e}^{-b_2 \Tr \abs{\J} h_0 \exp\left(\frac{\eta}{\abs{K}-1}  \delta \Tl\right)}.
      \end{align*}
      Since the initial concentration of $W^{j\to k}$ is $\frac{b_1}{b_2 \abs{\J}}$, and since it does not evolve during $[T^0_1, T^1_1]$, both $w^{j\to k}(T^1_1)$ and $\Bar{w}^{i\to k}_2$ are bounded by $\frac{b_1}{b_2}$, so we have
      \begin{align*}
        \abs{w^{j\to k}(T^2_1) - \Bar{w}^{j\to k}_2} &\leq 2\frac{b_1}{b_2}\mathrm{e}^{-b_2 \Tr\abs{\J}h_0 \exp(\eta (\abs{K}-1)^{-1} \delta \Tl)}.
      \end{align*}
      Since the exponential term is less than $1$ we also have $\abs{w^{i\to k}(T^2_1) - \Bar{w}^{i\to k}_2}\leq 2\frac{b_1}{b_2}$, so Eq.~\eqref{eq recursion} is true for $m=1$.

    \item Case $m>1$: Suppose Eq.~\eqref{eq recursion} holds for $m-1$. According to Eq.~\eqref{eq min w}, we have
      \begin{align*}
        &\abs{w^{j\to k}(T^2_m) - \Bar{w}^{j\to k}_{m+1}} \\
        &\leq  \abs{w^{j\to k}(T^1_m) - \Bar{w}^{j\to k}_{m+1}}\mathrm{e}^{-b_2 \Tr \abs{\J} h_0 \exp\left(\frac{\eta}{\abs{K}-1} m \delta \Tl\right)}\\
        &\leq ( \abs{w^{j\to k}(T^1_m) - \Bar{w}^{j\to k}_{m}} + \abs{\Bar{w}^{j\to k}_m - \Bar{w}^{j\to k}_{m+1}} )\mathrm{e}^{-b_2 \Tr \abs{\J} h_0 \exp\left(\frac{\eta}{\abs{K}-1} m \delta \Tl\right)} .
      \end{align*}
      Since Eq.~\eqref{eq recursion} holds for $m-1$, and since $w^{j\to k}(T^1_m) = w^{j\to k}(T^2_{m-1})$, the terms $\abs{w^{j\to k}(T^1_m) - \Bar{w}^{j\to k}_{m}}$ and $ \abs{\Bar{w}^{j\to k}_m - \Bar{w}^{j\to k}_{m+1}}$ are each bounded by $2\frac{b_1}{b_2}$, because by definition all weights $(\Bar{w}^{j\to k}_l)_{l\leq 0}$ are bounded by $\frac{b_1}{b_2}$. Therefore
      \begin{align*}
        \abs{w^{i\to k}(T^2_m) - \Bar{w}^{i\to k}_{m+1}} &\leq 4 \frac{b_1}{b_2} \mathrm{e}^{-b_2 \Tr \abs{\J} h_0 \exp\left(\frac{\eta}{\abs{K}-1} m \delta \Tl\right)}.
      \end{align*}
      Furthermore, the condition on $\Tr$ imposes that \[2 \mathrm{e}^{-b_2 \Tr \abs{\J} h_0 \exp\left(\frac{\eta}{\abs{K}-1} m \delta \Tl\right)} \leq 1,\]
      so we have as well $ \abs{w^{j\to k}(T^2_m) - \Bar{w}^{j\to k}_{m+1}} \leq 2\frac{b_1}{b_2}$.
  \end{itemize}

  This proves that Eq.~\eqref{eq recursion} is true for every $m\in [M]$, which implies Proposition \ref{prop cvg EWA} by lower-bounding $\exp\left(\frac{\eta}{\abs{K}-1} m \delta \Tl\right)$ by $1$.

  \subsubsection{Proof of Proposition \ref{prop reg}}\label{sec:proof prop reg}
  Let $C_6:= \max\left\{\ln\left(\frac{b_2}{b_1}\norm{f}_\infty \Tsel\right)(b_2a(0))^{-1},\, \ln(2) (b_2 h_0)^{-1}, \ln(2)a_2^{-1}\right\}$.

  \noindent Let $E(\Tr):= 4 \frac{b_1}{b_2} \mathrm{e}^{-b_2 \Tr\abs{\J}h_0}$ and $A :=2 a_1  \Tl s_0(\frac{b_1}{b_2} + \abs{\J} E(\Tr))$. We prove the following lemma by recursion:

  \begin{lemma} \label{lemma xk}
    Let $k\in K$. Suppose the assumptions of Proposition \ref{prop reg} hold. Then, for all $m\in [M]$,
    \begin{align}
      &\Abs{x^k(T^1_m) - \int_{T^0_m}^{T^1_m} \sum_{j\in \J} \Bar{w}^{j\to k}_m x^j(t) dt} \leq  A\mathrm{e}^{-a_2\Tr} +a_1 \Tl s_0 \abs{\J} E(\Tr) \label{eq recursion 2.1}
    \end{align}
    and
    \begin{align}
      &x^k(T^1_m) \leq A. \label{eq recursion 2.2}
    \end{align}
  \end{lemma}

  \begin{proof}
    \begin{itemize}
      \item Case $m=1$: Since $x^k(T^0_1)=0$ according to Eq.~\eqref{eq xk}, and since the initial concentration of $W^{j\to k}$ is equal to $\Bar{w}^{j\to k}_1$ and does not change during $[T^0_1, T^1_1]$, we have
        \begin{align*}
          x^k(T^1_1) &= a_1\int_{T^0_1}^{T^1_1} \sum_{j\in \J} \Bar{w}^{j\to k}_1 x^j(t) \, \mathrm{d}t\, .
        \end{align*}
        Moreover, $x^k(T^1_1) \leq a_1 \Tl s_0 \frac{b_1}{b_2} \leq A$ since $ \sum_{j\in \J} \Bar{w}^{j\to k}_1 = \frac{b_1}{b_2}$ and since $x^j$ is bounded by $s_0$. Therefore Eqs.~\eqref{eq recursion 2.1} and \eqref{eq recursion 2.2} are true for $m=1$.

      \item Case $m>1$: Suppose Eqs.~\eqref{eq recursion 2.1} and \eqref{eq recursion 2.2} are true for $m-1$. According to Eq.~\eqref{eq xk} and the proof of Proposition \ref{prop cvg EWA}, we have
        \begin{align*}
          & \Abs{x^k(T^1_m) - a_1\int_{T^0_m}^{T^1_m} \sum_{j\in \J} \Bar{w}^{j\to k} x^j(t) \,\mathrm{d}t} \\
          &\leq x^k(T^0_m) + a_1\Abs{ \int_{T^0_m}^{T^1_m} \sum_{j\in \J} (\Bar{w}^{j\to k} - w^{j\to k}) x^j(t) \,\mathrm{d}t} \\
          &\leq x^k(T^0_m) + a_1 \Tl s_0 \abs{\J} E(\Tr)\, .
        \end{align*}
        Besides, according to Eq.~\eqref{eq xk 2} and since $x^k(T^0_m) = x^k(T^2_{m-1})$, we have
        \begin{align*}
          x^k(T^0_m) &=  x^k(T^1_{m-1})\mathrm{e}^{-a_2\Tr}.
        \end{align*}
        Since Eq.~\eqref{eq recursion 2.2} is true for $m-1$, we have
        \begin{align*}
          \Abs{x^k(T^1_m) - a_1\int_{T^0_m}^{T^1_m} \sum_{i\in I} \Bar{w}^{i\to k} x^i(t) \,\mathrm{d}t} &\leq A\mathrm{e}^{-a_2\Tr} + a_1 \Tl s_0 \abs{\J} E(\Tr).
        \end{align*}
        Furthermore,
        \begin{align*}
          x^k(T^1_m)  &\leq a_1\int_{T^0_m}^{T^1_m} \sum_{j\in \J} \Bar{w}^{j\to k} x^j(t) \,\mathrm{d}t + A\mathrm{e}^{-a_2\Tr} +  a_1 \Tl s_0 \abs{\J} E(\Tr) \\
          &\leq a_1 \Tl s_0 \frac{b_1}{b_2} + A\mathrm{e}^{-a_2\Tr} +a_1 \Tl s_0 \abs{\J} E(\Tr) \\
          &\leq \frac{A}{2} + A\mathrm{e}^{-a_2\Tr}.
        \end{align*}
    \end{itemize}

    The condition on $\Tr$ implies that $\mathrm{e}^{-a_2\Tr} \leq \frac{1}{2}$ so Eqs.~\eqref{eq recursion 2.1} and \eqref{eq recursion 2.2} are true at rank $m$ as well.
  \end{proof}

  Therefore, according to Lemma \ref{lemma xk}, for $m\in [M]$ and $k\in K$ we have
  \begin{align*}
    &\Abs{x^k(T^1_m) - a_1\int_{T^0_m}^{T^1_m} \sum_{j\in \J} \Bar{w}^{j\to k}_m \Bar{x}^j(t) \,\mathrm{d}t} \\
    &\leq \Abs{x^k(T^1_m) - a_1\int_{T^0_m}^{T^1_m} \sum_{j\in \J} \Bar{w}^{j\to k}_m x^j(t) \,\mathrm{d}t} + \Abs{ a_1\int_{T^0_m}^{T^1_m} \sum_{j\in \J} \Bar{w}^{j\to k}_m (\Bar{x}^j(t) - x^j(t)) \,\mathrm{d}t } \\
    & \leq  A\mathrm{e}^{-a_2\Tr} +a_1 \Tl s_0 \abs{\J} E(\Tr) + a_1 \Tl \sup_{j\in \J, t\in [T^0_m, T^1_m]} \abs{x^j(t) - \Bar{x}^j(t)}\, .
  \end{align*}
  We bound $\sup_{j\in \J, t\in [T^0_m, T^1_m]} \abs{x^j(t) - \Bar{x}^j(t)}$ for $t\in  [T^0_m, T^1_m]$, $j \in \J$, as:
  \begin{align} \label{eq bound xj}
    \abs{x^j(t) - \Bar{x}^j(t)} &= \Phi^j(t) \AAbs{\frac{b_1}{b_2} - w^j(T^0_1)}  \leq \alpha  \max\left\{\frac{b_1}{b_2}, \norm{f}_{\infty} \Tsel\right\} \mathrm{e}^{-b_2 a_0 \Tr }\, ,
  \end{align}
  where the inequality holds because of Assumption \ref{assump delta} and the proof of Proposition \ref{prop bound wj}. Therefore, we get
  \begin{align*}
    \Abs{x^k(T^1_m) - a_1\int_{T^0_m}^{T^1_m} \sum_{j\in \J} \Bar{w}^{j\to k} \Bar{x}^j(t) \, \mathrm{d}t} \leq E(\Tl,\Tr) \, ,
  \end{align*}
  where
  \begin{align*}
    E(\Tl,\Tr) :=  A\mathrm{e}^{-a_2\Tr} &+a_1 \Tl s_0 \abs{\J} E(\Tr) \\
    &+ a_1 \Tl \alpha  \max\left\{\frac{b_1}{b_2}, \norm{f}_{\infty} \Tsel\right\} \mathrm{e}^{-b_2 a_0 \Tr }.
  \end{align*}
  Now that we have bounded the quantity $ \Abs{x^k(T^1_m) - a_1\int_{T^0_m}^{T^1_m} \sum_{j\in \J} \Bar{w}^{j\to k} \Bar{x}^j(t) \,\mathrm{d}t}$, we look at the species discrepancy of species $k$. According to its definition and the previous bound, we get
  \begin{align*}
    \AAbs{\disc^k - \disc^k(\Bar{w}^k)} &\leq 2\max_{m\in [M]}\AAbs{x^k(T^1_m)  - a_1\int_{T^0_m}^{T^1_m} \sum_{j\in \J} \Bar{w}^{j\to k}_m \Bar{x}^j(t) \, \mathrm{d}t} \\
    &\leq 2E(\Tl,\Tr) \, .
  \end{align*}

  To finish the proof, we use the regret bound of the expert aggregation EWA provided in Appendix \ref{appendix EWA}. For $k\in K$, the gains $g^{j\to k}_m$ belong to the interval $[0, \, 2\Tl s_0 \xi^{-1}]$. Therefore, the choice $\eta = \frac{\xi}{2\Tl s_0} \sqrt{8\frac{\ln(\abs{\J})}{M}}$ enables us to use Eq.~\eqref{eq regret EWA}. Let $E_{\text{reg}}(M) := 2\Tl s_0 \xi^{-1} \sqrt{\frac{M}{2}\ln(\abs{\J})}$. The regret of species $k$ using EWA, denoted $R^k_M$, is
  \[
    R^k_M := \max_{p \in \mathcal{P}_{\!\J}} \sum_{m=1}^M \sum_{j\in \J} p^{j\to k} g^{j\to k}_m -  \sum_{m=1}^M \sum_{j\in \J} q_m^{j\to k} g^{j\to k}_m \, ,
  \]
  where $q^{j\to k}_m = \frac{b_2}{b_1} \Bar{w}^{j\to k}_m$. According to Eq.~\eqref{eq regret EWA}, it satisfies
  \begin{equation} \label{eq reg k proof}
    R^k_M \leq E_{\text{reg}}(M).
  \end{equation}
  For $j\in \J$, $k\in K$ let
  \begin{equation} \label{eq def g bar}
    \Bar{g}^{j\to k}_m :=
    \begin{cases}
      \int_{T^0_m}^{T^1_m}(\Bar{x}^j(u) + s_0)\, \mathrm{d}u \times \frac{M}{M^k} &\text{if } o_m\in k  \\
      \int_{T^0_m}^{T^1_m}(s_0 - \Bar{x}^j(u))\, \mathrm{d}u\!\times\! \frac{M}{M^{k'}} \!\times\! \frac{1}{\abs{K}-1} &\text{if } o_m \in k'\neq k
    \end{cases}
  \end{equation}
  and
  \[
    \Bar{R}^k_M := \max_{p \in \mathcal{P}_{\!\J}} \sum_{m=1}^M \sum_{j\in \J} p^{j\to k} \Bar{g}^{j\to k}_m -  \sum_{m=1}^M \sum_{j\in \J} q_m^{j\to k} \Bar{g}^{j\to k}_m .
  \]
  For $j\in \J$, let us bound $ \AAbs{\Bar{g}^{j\to k}_m -  g^{j\to k}_m}$ in order to bound $\Bar{R}^k_M$. According to the definitions of $g^{j\to k}_m$, $\Bar{g}^{j\to k}_m$, Assumption \ref{assump xi}, and Eq.~\eqref{eq bound xj}, we have
  \begin{align*}
    \AAbs{\Bar{g}^{j\to k}_m -  g^{j\to k}_m} \leq \Tl \xi^{-1} \alpha  \max\left\{\frac{b_1}{b_2},\, \norm{f}_{\infty} \Tsel\right\} \mathrm{e}^{-b_2 a_0 \Tr }.
  \end{align*}
  Therefore
  \begin{align} \label{eq reg bar}
    \Bar{R}^k_M &\leq \AAbs{\Bar{R}^k_M - R^k_M} + R^k_M \leq 2M \Tl \xi^{-1} \alpha  \max\left\{\frac{b_1}{b_2},\, \norm{f}_{\infty} \Tsel\right\} \mathrm{e}^{-b_2 a_0 \Tr} + E_{\text{reg}}(M).
  \end{align}
  Besides, by definition of the species discrepancy and the species discrepancy with constant weights,  we have in fact

  \begin{align*}
    & \frac{1}{M}  \sum_{m=1}^M \sum_{j\in \J} \Bar{w}_m^{j\to k} \Bar{g}^{j\to k}_m  \\
    &= \frac{1}{M^k}\sum_{m,\ o_m\in k}\sum_{j\in \J} \Bar{w}_m^{j\to k}\int_{T^0_m}^{T^1_m} (\Bar{x}^j(u) + s_0)\,\mathrm{d}u \\
    &\hspace{1cm} + \frac{1}{\abs{K}-1}\sum_{k'\neq k} \frac{1}{M^{k'}} \sum_{m, \ o_m\in k'}\sum_{j\in \J}\Bar{w}_m^{j\to k}\int_{T^0_m}^{T^1_m} ( - \Bar{x}^j(u) + s_0)\,\mathrm{d}u \\
    & = \Brac{\sum_{j\in \J} \Bar{w}_m^{j\to k}\int_{T^0_m}^{T^1_m} \Bar{x}^j(u)du }_{m, \ o_m\in k} - \Brac{\sum_{j\in \J} \Bar{w}_m^{j\to k}\int_{T^0_m}^{T^1_m} \Bar{x}^j(u)\,\mathrm{d}u }_{
      \begin{subarray}{l}
        k'\neq k \\ m, \ o_m\in k'
    \end{subarray}}
    \\
    &\hspace{1cm}  + \Brac{\Tl s_0}_{m, \ o_m\in k}+ \Brac{\Tl s_0}_{
      \begin{subarray}{l}
        k'\neq k \\ m, \ o_m\in k'
    \end{subarray}} \\
    &= \disc^k(\Bar{w}^k) + 2 \Tl s_0
  \end{align*}
  and similarly for $q^k \in \frac{b_1}{b_2}\mathcal{P}_{\J}$,
  \[
    \frac{1}{M}  \sum_{m=1}^M \sum_{j\in \J} q^{j\to k} \Bar{g}^{j\to k}_m =  \disc^k(q^k) + 2 \Tl s_0 \, .
  \]
  Multiplying Eq.~\eqref{eq reg bar} by $\frac{b_1}{b_2M}$, we get
  \begin{align*}
    &\max_{q \in \frac{b_1}{b_2}\mathcal{P}_{\!\J}}  \disc^k(q^k) - \disc^k(\Bar{w}^k) \\
    &\leq 2\frac{b_1}{b_2} \Tl \xi^{-1} \alpha  \max\left\{\frac{b_1}{b_2}, \,\norm{f}_{\infty} \Tsel\right\} \mathrm{e}^{-b_2 a_0 \Tr}  + \frac{b_1}{b_2M} E_{\text{reg}}(M)\, .
  \end{align*}
  Finally, we get
  \[
    \max_{q^k \in \frac{b_1}{b_2}\mathcal{P}_{\!\J}}  \disc^k(q^k) - \disc^k \leq E(\Tl,\Tr,M)\, ,
  \]
  where
  \begin{align*}
    E(\Tl,\Tr,M) := 2\frac{b_1}{b_2} &\Tl \xi^{-1} \alpha  \max\left\{\frac{b_1}{b_2},\, \norm{f}_{\infty} \Tsel\right\} \mathrm{e}^{-b_2 a_0 \Tr}  \\
    & + \frac{b_1}{b_2M} E_{\text{reg}}(M) + 2E(\Tl,\Tr).
  \end{align*}
  We get the final result by writing only the dependency in $\abs{\J}$, $\Tr$, and $M$.

  \subsubsection{Proof of Theorem \ref{th oracle}}\label{sec:proof th oracle}

  We start by noticing that
  \begin{align*}
    \Disc &=  \Brac{x^{k^\star}(T^1_m) - x^{k}(T^1_m)}_{
      \begin{subarray}{l}
        k^\star \in K \\
        k \neq k^\star \\
        m, \ o_m \in k^\star
    \end{subarray}}\\
    & = \Brac{x^{k^\star}(T^1_m)}_{
      \begin{subarray}{l}
        k^\star \in K \\
        m, \ o_m \in k^\star
    \end{subarray}} - \Brac{x^{k}(T^1_m)}_{
      \begin{subarray}{l}
        k^\star \in K \\
        k \neq k^\star \\
        m, \ o_m \in k^\star
    \end{subarray}}\, .
  \end{align*}
  Let us exchange the names of the indices $k^\star$ and $k$ in the second term, as well as the sums over $k$ and over $k^\star$:
  \begin{align*}
    \Disc & = \Brac{x^{k^\star}(T^1_m)}_{
      \begin{subarray}{l}
        k^\star \in K \\
        m, \ o_m \in k^\star
    \end{subarray}} - \Brac{x^{k^\star}(T^1_m)}_{
      \begin{subarray}{l}
        k^\star \in K \\
        k \neq k^\star \\
        m, \ o_m \in k
    \end{subarray}}\\
    &= \Brac{\Brac{x^{k^\star}(T^1_m)}_{
        \begin{subarray}{l}
          m, \ o_m \in k^\star
      \end{subarray}} - \Brac{x^{k^\star}(T^1_m)}_{
        \begin{subarray}{l}
          k \neq k^\star \\
          m, \ o_m \in k
    \end{subarray}}}_{k^\star \in K} \\
    & = \Brac{\disc^k }_{k^\star\in K}.
  \end{align*}
  Let $q^K \in (\frac{b_1}{b_2}\mathcal{P}_{\J})^{\abs{K}}$. According to the proof of Proposition \ref{prop reg},
  \begin{align*}
    \Disc \geq  \Brac{ \disc^k(q^k)  }_{k\in K} - E(\Tl,\Tr,M)\, ,
  \end{align*}
  where $E(\Tl,\Tr,M)$ is given in the proof of Proposition \ref{prop reg}. With the same computation as before, we get $\Brac{ \disc^k(q^k)  }_{k\in K}  = \Disc(q^K)$. Therefore
  \begin{align*}
    \Disc \geq  \Disc(q^K) - E(\Tl,\Tr,M)\, .
  \end{align*}
  This holds for every $q^K$, therefore proving the result.

  \subsubsection{Proof of Theorem \ref{theo lim wk}}\label{sec:proof theo lim wk}

  For $k\in K$ and $j\in \J$, let $\bbw^{j\to k}_{M+1} := \frac{b_1}{b_2} \frac{\exp(\eta \Bar{G}^{j\to k}_M)}{\sum_{j'\in \J} \exp(\Bar{G}^{j'\to k}_M)}$, where for $j'\in \J$, we define $\Bar{G}^{j'\to k}_M := \sum_{m=1}^M \Bar{g}^{j\to k}_m$, were $\Bar{g}^{j\to k}_m$ is given by Eq.~\eqref{eq def g bar}. For $k\in K$, a family $(a^{j\to k})_{j\in \J}$ is denoted $a^k$. Then we have
  \begin{align*}
    \Norm{w^{k}(T^2_M) - \bw^{k}}_2 &\leq \Norm{w^{k}(T^2_M) - \bw^{k}_{M+1}}_2 +  \Norm{\bw^{k}_{M+1} - \bbw^{k}_{M+1}}_2 + \Norm{\bbw^{k}_{M+1} - \bw^{k}}_2.
  \end{align*}
  We already bounded $\norm{w^{k}(T^2_M) - \bw^{k}_{M+1}}_2$ in Proposition \ref{prop cvg EWA}. Let us bound the two remaining terms.
  \newline

  \begin{itemize}
    \item Bound of $\Norm{\bw^{k}_{M+1} - \bbw^{k}_{M+1}}_2$: Let $j\in \J$. Let us start by bounding $\AAbs{\Bar{G}^{j\to k}_M - G^{j\to k}_M}$. According to the definition of the gains $g^{j\to k}_m$, since the concentration of the species $W^j$ does not evolve during the learning phase, we get
      \begin{align*}
        G^{j\to k}_M  = w^j(T^1_0) &\left(\frac{M}{M^k} \sum_{m, \ o_m \in k} \int_{T^0_m}^{T^1_m} \Phi^j(u) \,\mathrm{d}u - \frac{1}{\abs{K} - 1}\sum_{k'\neq k} \frac{M}{M^{k'}} \sum_{m, \ o_m\in k'} \int_{T^0_m}^{T^1_m}\Phi^j(u) \,\mathrm{d}u\right) \\
        &+ 2M s_0 \Tl \, .
      \end{align*}
      Furthermore, for $k'\in K$, according to Assumption \ref{assump natures and nb obj}, we have $M^{k'} = \frac{M\abs{k}}{\abs{\Obj}}$ and
      \begin{align*}
        \sum_{m, \ o_m\in k'} \int_{T^0_m}^{T^1_m}\Phi^j(u) \,\mathrm{d}u &= \sum_{o\in k'} \sum_{m, \ o_m = o} \int_o\Phi^j = \sum_{o\in k} \frac{M}{\abs{\Obj}} \int_o \Phi^j.
      \end{align*}
      Thus, according to the definition of the flux discrepancy (Definition \ref{def feat disc}), we get
      \begin{align*}
        G^{j\to k}_M
        & =  w^j(T^1_0) M \FDisc^{j\to k} + 2M s_0 \Tl\, .
      \end{align*}
      With the same computation, according to the definition of the gains $\Bar{g}^{j\to k}_m$, we have
      \begin{align} \label{eq bG}
        \Bar{G}^{j\to k}_M & =  \frac{b_1}{b_2} M \FDisc^{j\to k} + 2M s_0 \Tl \, .
      \end{align}
      Hence, according to the proof of Proposition \ref{prop bound wj}, we get
      \begin{align*}
        \AAbs{\Bar{G}^{j\to k}_M - G^{j\to k}_M} & = \AAbs{w^j(T^0_1) - \frac{b_1}{b_2}} \FDisc^{j_S\to k}M \\
        &\leq 2\alpha \Tl \max\left\{\frac{b_1}{b_2},\, \norm{f}_{\infty} \Tsel\right\} M \mathrm{e}^{-b_2 a_0 \Tr },
      \end{align*}
      where we bounded $\FDisc^{j\to k}$ by $2\alpha \Tl$. We can combine this bound with Proposition \ref{prop iaf} to get
      \begin{align*}
        \Norm{\bw^{k}_{M+1} - \bbw^{k}_{M+1}}_2 &\leq \eta \norm{\Bar{G}^k_M - G^k_M}_2 \\
        &\leq \eta \sqrt{\abs{\J}} 2\alpha \Tl \max\left\{\frac{b_1}{b_2}, \, \norm{f}_{\infty} \Tsel\right\} M \mathrm{e}^{-b_2 a_0 \Tr }\\
        & = \xi \alpha s_0^{-1}  \max\left\{\frac{b_1}{b_2},\, \norm{f}_{\infty} \Tsel\right\}\sqrt{8 \abs{\J}\ln(\abs{\J})M} \mathrm{e}^{-b_2 a_0 \Tr } .
      \end{align*}

    \item Bound of $\Norm{\bbw^{k}_{M+1} - \bw^k}_2$: According to Eq.~\eqref{eq bG}, for $j\in \J$, $k\in K$ we have
      \begin{align*}
        \eta \Bar{G}^{j\to k}_M &= \xi \left( 1 + \frac{b_1}{2b_2 s_0 \Tl} \FDisc^{j\to k} \right) \sqrt{8\ln(\abs{\J}) M}\, .
      \end{align*}
      We apply Proposition \ref{prop prelim EWA} and distinguish two cases:
      \begin{enumerate}
        \item If $\J^k = \J$ then $\bbw^k_{M+1} = \bw^k$.
        \item If $\J^k \subsetneq \J$ then let $\Delta^k := \max_{j\in \J} \FDisc^{j\to k} - \max_{j\notin \J^k} \FDisc^{j\to k}$. We then have
          \begin{align*}
            &\Norm{\bbw^{k}_{M+1} -  \bw^k}_2 \\
            &\leq \frac{b_2}{b_1}\frac{\abs{\J}^{1/2}}{\abs{\J^k}}\max\left\{1,\frac{\abs{\J}- \abs{\J^k}}{\abs{\J^k}}\right\} \exp\left(-\frac{\xi b_1 }{2b_2 s_0 \Tl} \Delta^k \sqrt{8\ln(\abs{\J}) M}\right).
          \end{align*}
      \end{enumerate}
  \end{itemize}

  We get the result by adding all error terms.

  \subsubsection{Proof of Theorem \ref{th equiv}}\label{sec:proof th equiv}

  We assume that Assumption \ref{assump bin flux} holds.

  \begin{itemize}
    \item Proof of $(i)\implies (iii)$: Suppose Assumption \ref{assump class dec} (class decomposition) holds. Let us compute the flux discrepancies of the species in $\J$ in order to compute the weights $\bw^K$ given in Theorem \ref{theo lim wk}. For every $k\in K$, we choose the set $E^k$ given by Assumption \ref{assump xi} with maximal size. Let $j\in \J$. We recall that
      \begin{align*}
        \FDisc^{j\to k} &= \Brac{\int_o\Phi^j}_{o\in k} - \Brac{\int_o \Phi^j}_{
          \begin{subarray}{l}
            k'\neq k \\
            o\in k'
        \end{subarray}} .
      \end{align*}
      According to Assumption \ref{assump bin flux} (binary correlations), $\int_o \Phi^j = p$ if and only if $o\in \Obj^j$, and $\int_o \Phi^j = 0$ otherwise.

      Let $j\in \J$ and $o\in k$. We have $\Brac{\int_o\Phi^j}_{o\in k} =  \frac{Z}{\abs{k}} p$, where $Z$ is the common value of all $\abs{\Obj^{j'}}$ for $j'\in \J$. Let $o\in k'\neq k$. If $\int_o \Phi^j>0 $, it would mean that $o$ has features $j$, so $o$ would belong to $\Obj^j$. Since $j\in E^k$, this would mean that $o$ would belong to class $k$, which is impossible. Therefore $\int_o \Phi^j = 0 $ for every $o\notin k$, $\Brac{\int_o \Phi^j}_{
        \begin{subarray}{l}
          k'\neq k \\
          o\in k'
      \end{subarray}} = 0$, and $\FDisc^{j\to k} = \frac{Z}{\abs{k}} p$.

      Let $j\in \J$ with $j\notin E^k$. This means that there exists $o\in \Obj^j$ such that $o\notin k$. Then, $\Brac{\int_o \Phi^j}_{o\in k} \leq  \frac{Z-1}{\abs{k}} p$ and consequently $\FDisc^{j\to k} \leq \frac{Z-1}{\abs{k}} p$.
      Hence, $\J^k = E^k$. So, according to Theorem \ref{theo lim wk}, the limit weights are $\bw^{j\to k} = \frac{1}{\abs{E^k}} \mathbb{1}_{j\in E^k}$.

      Now that we computed the limit weight family, we can verify that it is an optimal weight family. Let $k\in K$ and $o\in \Obj$.
      \begin{align*}
        \sum_{j\in \J} \bw^{j\to k} \int_o \Phi^j &=  \sum_{j\in \J} \frac{1}{\abs{E^k}} \mathbb{1}_{j\in E^k} \times p \mathbb{1}_{o\in k} = \frac{n^k_o p}{\abs{E^k}}\, ,
      \end{align*}
      where $n^k_o$ is the number of sets $j\in E^k$ such that $o$ has features $j$. If $o\in k$, according to Assumption \ref{assump class dec} (class decomposition), we have $n^k_o>0$ and therefore $ \sum_{j\in \J} \bw^{j\to k} \int_o \Phi^j>0$. If $o\notin k$, then $n^k_o = 0$ (otherwise $o$ would belong to a set $\Obj^j$ in the composition of class $k$, so $o$ would belong to $k$), and therefore  $\sum_{j\in \J} \bw^{j\to k} \int_o \Phi^j=0$. Hence, $\bw^K$ is an optimal weight family.
      \newline

    \item Proof of $(ii)\implies (i)$: Suppose there exists an optimal weight family $q^K$. Let $k\in K$ and $E^k := \{j\in \J \text{ such that }  q^{j\to k}>0\}$. Let us show that $k = \bigcup_{j\in E^k} \Obj^j$. Let $o\in k$. Since $q^K$ is an optimal weight family, we have $\sum_{j\in \J} q^{j\to k} \int_o \Phi^j >0$. Therefore, there exists $j \in \J$ such that $q^{j\to k}>0$ and $\int_o \Phi^j >0$. Hence, $j\in E^k$ and $o$ has features $j$ so $o\in \bigcup_{j\in E^k} \Obj^j$. Consequently, $k \subset \bigcup_{j\in E^k} \Obj^j$.

      Let $o\in \bigcup_{j\in E^k} \Obj^j$. There exists $j\in \J$ such that $o$ has features $j$ and $q^{j\to k}>0$. Then, $\sum_{j\in \J} q^{j\to k} \int_o \Phi^j >0$. Since $q^K$ is an optimal weight family, this implies that $o\in k$. Therefore, $\bigcup_{j\in E^k} \Obj^j \subset k$.
  \end{itemize}

  It is clear that $(iii)\implies (ii)$, so the three statements are equivalent.

  \subsubsection{Proof of Proposition \ref{prop vcdim}}\label{sec:proof prop vcdim}

  The VC-dimension is the size of the largest set $E$ than can be fully shattered by $\h$, \ie $\h$ can realize all possible dichotomies of $E$. Hence, to prove that $\operatorname{VCdim}(\h) = \abs{\J}$, we start by considering a set $E$ of size $\abs{\J}$ than can be fully shattered by $\h$, implying that $\operatorname{VCdim}(\h) \geq \abs{\J}$. Then, we prove that any set that can be fully shattered by $\h$ has a size less than or equal to $\abs{\J}$, leaving only the case $\operatorname{VCdim}(\h) = \abs{\J}$.
  \newline

  \begin{enumerate}
    \item Let $j\in \J$ and $o^j := (o^j_i)_{i\in I}$, where $o^i_j = 1$ if and only if $i\in j$, such that there is a one-to-one correspondence between sample $o^j$ and the set $j$. Let $E:= \{o^j, j\in \J\}$. Then $\abs{E}=\abs{\J}$. To show that this set can be fully shattered by $\h$, consider $E'\subset E$ and $F:= \{j\in \J,\, \exists o\in \E', \, o = o^j\}$. Then, $E'$ and $E\setminus E'$ are separated by $\mathbb{1}_F$. Hence, $\operatorname{VCdim}(\h) \geq \abs{\J}$.
      \newline

    \item Let $E$ be a set of sample types that can be fully shattered by $\h$. In particular, for each $o\in E$, there exists a non-empty $F\subset \J$ such that $o$ is the unique item of $E$ that activates a set $j\in F$. Let $j_o$ be such a species. Then $o\mapsto j_o$ is an injection, so $E\leq \abs{\J}$. Therefore, $\operatorname{VCdim}(\h) \leq \abs{\J}$.
  \end{enumerate}

  \vskip 0.2in

  \bibliography{biblio}

  \end{document}